\documentclass[pdflatex,sn-mathphys-num]{sn-jnl}

\usepackage{graphicx}%
\usepackage{multirow}%
\usepackage{amsmath,amssymb,amsfonts}%
\usepackage{amsthm}%
\usepackage[title]{appendix}%
\usepackage{textcomp}%
\usepackage{manyfoot}%
\usepackage{booktabs}%
\usepackage{algorithm}%
\usepackage{algorithmicx}%
\usepackage{algpseudocode}%
\usepackage{listings}%
\usepackage{hyperref}

\usepackage[table]{xcolor}

\definecolor{TableHead}{HTML}{D9D5C9}
\definecolor{TableStripe}{HTML}{F2EFE8}
\definecolor{TableRule}{HTML}{2B2B2B}

\theoremstyle{thmstyleone}%
\newtheorem{theorem}{Theorem}%
\newtheorem{proposition}[theorem]{Proposition}%

\theoremstyle{thmstyletwo}%

\theoremstyle{thmstylethree}%
\renewcommand{\thefigure}{S\arabic{figure}}
\renewcommand{\thetable}{S\arabic{table}}
\renewcommand{\thesection}{S\arabic{section}}
\renewcommand{\theequation}{S\arabic{equation}}

\newtheorem{lemma}[theorem]{Lemma}%
\newtheorem{corollary}[theorem]{Corollary}%
\begin{document}

\title[Scalar-pathway fidelity in equivariant potentials]{Scalar-pathway fidelity improves physical accuracy in short-range equivariant interatomic potentials}

\author[1]{\fnm{Jia} \sur{Bi}}\email{Jia.Bi@stfc.ac.uk}
\author*[2]{\fnm{Alin Marin} \sur{Elena}}\email{alin-marin.elena@stfc.ac.uk}
\author*[1,3]{\fnm{Samuel} \sur{Pinilla}}\email{samuel.pinilla@diamond.ac.uk}

\affil[1]{\orgdiv{Science and Technology Facilities Council}, \orgaddress{\street{Harwell Campus}, \city{Didcot}, \postcode{OX11 0QX}, \state{Oxfordshire}, \country{United Kingdom}}}

\affil[2]{\orgdiv{Science and Technology Facilities Council}, \orgaddress{\street{Keckwick Lane}, \city{Daresbury}, \postcode{WA4 4AD}, \state{Cheshire}, \country{United Kingdom}}}

\affil[3]{\orgdiv{Diamond Light Source}, \orgaddress{\street{Harwell Science and Innovation Campus}, \city{Didcot}, \postcode{OX11 0QX}, \state{Oxfordshire}, \country{United Kingdom}}}

\abstract{Accurate interatomic potentials enable molecular dynamics of materials, molecules, and interfaces beyond density-functional-theory length and time scales. Equivariant neural network potentials have improved the representation of local geometry. However, their deployable energy surfaces ultimately manifest through invariant scalar channels, whose aggregation and spectral resolution remain comparatively underexamined. Here we use Physics-Aware Neighborhood (PAN) pooling and Physics-Guided Spectral (PGS) mixers as controlled scalar-pathway probes: lightweight, symmetry-preserving modifications that act only on \(\ell=0\) channels while leaving the equivariant tensor backbone unchanged. Using MACE as a high-body-order mechanistic scaffold, PAN adds coordination-sensitive amplitude modulation, whereas PGS augments edge and readout scalar features with radial and tapered spectral bases. Across metallic Ag, covalent Si, a short-range ionic LiF/Li--F subset, and MD17/rMD17 molecules, this scalar-pathway correction reduces MACE force errors by 22--27\% and energy errors by 19--22\%; on systems with stress labels, stress errors decrease by 27--28\%, at approximately 5\% additional inference-FLOPs cost. Directionally consistent gains in Allegro and NequIP further indicate that the correction is portable across distinct short-range equivariant backbones, although effect sizes remain architecture-dependent. These results identify scalar-pathway fidelity as a practical design dimension for short-range equivariant interatomic potentials.}

\maketitle

\section{Introduction}
\label{sec_intro}
Atomic-scale simulations are foundational to materials science, chemistry, and condensed-matter physics because they connect microscopic structure to defect formation, phase transformations, transport, catalysis, and mechanical response. The predictive power of these simulations rests on the fidelity of the potential energy surface (PES), from which interatomic forces, stress, and long-time dynamical behavior are derived. First-principles methods, rooted in density-functional theory (DFT), provide a quantum-mechanical route to this surface but remain costly for routine simulations beyond hundreds of atoms and picosecond timescales~\cite{hohenberg1964inhomogeneous,kohn1965self,perdew1996generalized,blochl1994projector,kresse1996efficient,car1985unified,marx2009ab}. Classical force fields and molecular-dynamics algorithms enable much larger simulations. However, fixed analytic forms can struggle with complex bonding, metallic delocalization, anharmonicity, and changes in local coordination~\cite{allen1987computer,frenkel2002understanding,daw1984embedded,stillinger1985computer,tersoff1988new,brenner2002second,van2001reaxff,plimpton1995fast}.

Machine-learned interatomic potentials (MLIPs) bridge this divide by learning the PES directly from electronic-structure data, bringing near-\textit{ab initio} accuracy to molecular dynamics at a fraction of the cost~\cite{behler2007generalized,bartok2010gaussian,smith2017ani,thompson2015spectral,shapeev2016mtp,zhang2018deep,dral2020quantum,unke2021machine}. This progress has expanded from system-specific potentials to large materials datasets, foundation-style atomistic models and benchmark workflows for discovery and simulation~\cite{chanussot2021open,chen2022m3gnet,deng2023chgnet,merchant2023scaling,batatia2024foundation,yang2024mattersim,neumann2024orb,park2024sevennet,riebesell2025matbench}. A major architectural driver has been geometric deep learning~\cite{bronstein2021geometric}: equivariant networks represent local atomic geometry with features that transform predictably under rotations and reflections, allowing scalar energies and vector forces to be learned while respecting Euclidean symmetry~\cite{thomas2018tensor,anderson2019cormorant,fuchs2020se3,geiger2022e3nn}. In atomistic modeling, this direction includes SchNet~\cite{schutt2017schnet}, DimeNet/DimeNet++~\cite{klicpera2020dimenet,klicpera2020dimenetpp}, GemNet/GemNet-OC~\cite{gasteiger2021gemnet}, PaiNN~\cite{schutt2021equivariant}, NewtonNet~\cite{haghighatlari2022newtonnet}, SpookyNet~\cite{unke2021spookynet}, NequIP~\cite{batzner2022e3}, Allegro~\cite{musaelian2023learning}, EquiformerV2~\cite{liao2024equiformerv2}, and the atomic cluster expansion (ACE) and its neural extensions~\cite{drautz2019atomic,kovacs2023neural}. Among short-range equivariant MLIPs, MACE, a higher-order equivariant message-passing architecture related to neural atomic cluster expansions,~\cite{batatia2022mace,birch2024mace} is particularly notable for capturing high-body-order geometric correlations while retaining analytic energy, force, and stress derivatives needed for stable molecular dynamics. We use MACE as the primary mechanistic scaffold because this architecture imposes a stringent scalar-compression test: high-body-order equivariant tensor features are generated upstream, while invariant scalar channels still mediate the deployment-relevant energy, force, and stress predictions. Allegro and NequIP are then used as architecture-transfer controls rather than leaderboard baselines to test whether the same scalar-channel correction remains beneficial when the surrounding equivariant architecture changes.

These architectures share a common structural feature that has received comparatively little scrutiny. Although equivariant layers propagate rich tensorial representations of local geometry, deployment-relevant predictions are ultimately obtained by projecting this information onto invariant scalar channels. This is a representation-compression problem: if the scalar pathway lacks geometric adaptivity or spectral resolution, information encoded by the equivariant backbone can be weakened before it reaches the supervised energy output, regardless of how expressive the tensor representation is. Analogous scalar readout stages appear in many equivariant models. However, in this work, we restrict the evidence and claims to short-range equivariant MLIPs, where the scalar pathway is directly tied to forces, stress, and molecular-dynamics observables.

In short-range equivariant MLIPs, this scalar pathway takes a concrete, analytically tractable form. Two systemic limitations arise. First, scalar messages are commonly aggregated by an unweighted permutation-invariant sum. The messages being summed are themselves learned distance-, species-, and feature-dependent quantities; the limitation is that the final scalar aggregation lacks an explicit adaptive amplitude gate conditioned on local coordination, density, or geometric distortion. This can contribute to geometric under-resolution in heterogeneous environments such as surfaces, defect-adjacent regions, strained neighborhoods, and disordered phases~\cite{zeni2023stability,fu2022forces,gilmer2017neural,gasteiger2021gemnet,alon2021on}. Second, scalar nonlinearities such as SiLU exhibit a strong low-frequency spectral bias~\cite{rahaman2019spectral,tancik2020fourier,basri2020frequency,sitzmann2020implicit,sirignano2018dnnfail}, suppressing high-frequency structure in the PES and contributing to characteristic physical errors: softened short-range repulsion, red-shifted optical phonon modes, and energy drift under finite-timestep microcanonical integration~\cite{wang2021implicitbias}. In the short-range MACE setting examined here, these limitations lead to four distinct, quantifiable physical signatures that we identify and systematically target.

The scope of this test is deliberately short-range. PAN--PGS is not designed to replace explicit electrostatics, charge equilibration, polarizability, dispersion, or reactive chemistry; instead, it targets the residual short-range error that remains after an equivariant local representation has been built and before that information is expressed through the scalar energy pathway. The intervention is therefore complementary to long-range or reactive MLIP extensions rather than an alternative to them~\cite{thole1981molecular,rappe1991charge,grimme2010consistent,ko2021fourthgen,wang2021dplr,kapil2022nequip_longrange}.

Rather than treating the scalar pathway as a fixed implementation detail, we reframe it as a controllable design axis: the distributed set of scalar operations through which an equivariant representation is aggregated, spectrally transformed, and converted into the total energy and its derivatives. The central question we pose is: \textit{to what extent do residual physical errors in short-range MLIPs arise in this scalar pathway, rather than from the tensor backbone alone?} To probe this, we introduce two lightweight, symmetry-preserving modules that act exclusively on $\ell{=}0$ channels and leave all $O(3)$-equivariant tensor-product operations unchanged. Physics-Aware Neighborhood (PAN) pooling replaces uniform aggregation with amplitude modulation by sigmoid gates derived from invariant pairwise descriptors. Physics-Guided Spectral (PGS) mixers inject physically motivated spectral bases---a Fourier--Bessel expansion on pair distances at the edge level, and a tapered Exponential-of-Semicircle expansion on per-atom invariant latents at readout. PAN and PGS therefore serve as controlled probes of a shared scalar interface: the MACE experiments isolate the mechanism under a fixed high-body-order scaffold, whereas the Allegro and NequIP insertions test architecture-level portability across different scalar-processing interfaces.

Benchmarked across metallic Ag, covalent Si, a short-range ionic LiF/Li--F subset, and MD17/rMD17 molecular datasets~\cite{Christensen2020rmd17, Chmiela2017MD17}, integrating PAN and PGS into MACE yields systematic reductions in energy, force, stress, and dynamical errors at approximately 5\% additional inference-FLOPs overhead, reproducibly across five independent training seeds. The same improvements transfer directionally to two further equivariant backbones, Allegro and NequIP, under matched protocols. We therefore present scalar-pathway fidelity not as an alternative to tensor expressivity, but as an orthogonal design dimension that shapes how much of an equivariant representation is preserved when it becomes an energy landscape.

\section{Results}
\label{sec:results}

\subsection{A scalar compression interface in short-range equivariant MLIPs}
\label{subsec:framework}

\begin{figure*}[!t]
    \centering
    \includegraphics[width=0.98\textwidth]{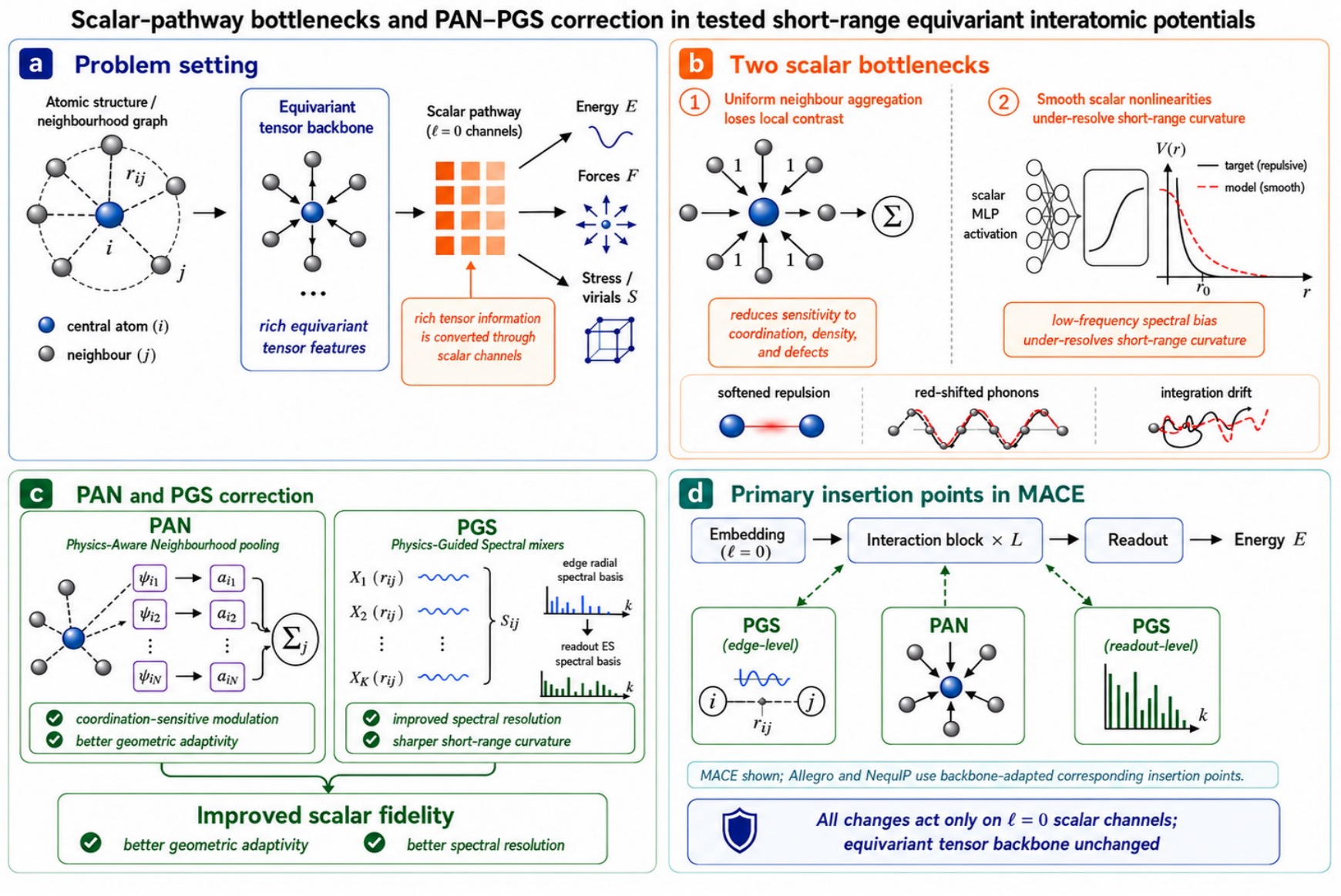}
    \caption{\textbf{Scalar-pathway compression interface and PAN--PGS correction in the tested short-range equivariant setting.}
    \textbf{(a)} Short-range equivariant interatomic potentials build rich tensorial representations, but the final energy and its derivatives are computed from invariant scalar channels.
    \textbf{(b)} Two scalar bottlenecks considered in this work: uniform neighbor aggregation, which weakens sensitivity to coordination, density, and defect structure, and smooth scalar nonlinearities, which under-resolve short-range curvature through low-frequency spectral bias.
    \textbf{(c)} PAN introduces geometry-aware scalar aggregation, whereas PGS enriches scalar spectral resolution.
    \textbf{(d)} PAN and PGS are inserted only on $\ell{=}0$ scalar channels in MACE, leaving the equivariant tensor-product backbone unchanged.}
    \label{fig:pan_pgs_framework}
\end{figure*}

Figure~\ref{fig:pan_pgs_framework} translates the scalar-pathway hypothesis into the controlled experimental design used throughout this work. The primary MACE experiments implement the most stringent version of this test: high-body-order equivariant tensor features are constructed upstream, while the deployed energy, forces, and stress are recovered through invariant \(\ell=0\) scalar channels. PAN and PGS are inserted only at this scalar interface, leaving the equivariant tensor-product backbone, interaction depth, cutoff, training budget, and data splits unchanged.

The two interventions target complementary points in this interface. PAN tests whether scalar aggregation benefits from explicit coordination- and
distortion-sensitive amplitude modulation. PGS tests whether scalar readout benefits from a richer short-range spectral basis at the edge and final invariant readout levels. Because both modules preserve permutation symmetry, \(O(3)\)-equivariance of forces and differentiability of the scalar energy, any measured change can be attributed to how invariant scalar information is aggregated and spectrally transformed rather than to changes in the equivariant backbone itself.

This design yields three linked tests. First, scalar-pathway fidelity should improve energy, force, and stress accuracy without the cost of enlarging the tensor backbone. Second, PAN and PGS should show separable signatures: PAN should contribute most in distorted or under-coordinated environments, whereas PGS should improve short-range repulsive curvature and stiff local modes. Third, the correction should attenuate physical observables controlled by these scalar failures, including softened short-range repulsion, red-shifted optical modes, finite-step NVE drift, and liquid over-approach. The following sections evaluate these predictions in sequence.

\subsection{Scalar-channel corrections: accuracy gains, minimal overhead}
\label{subsec:headline_results}

\begin{figure*}[!t]
    \centering
    \includegraphics[width=0.98\textwidth]{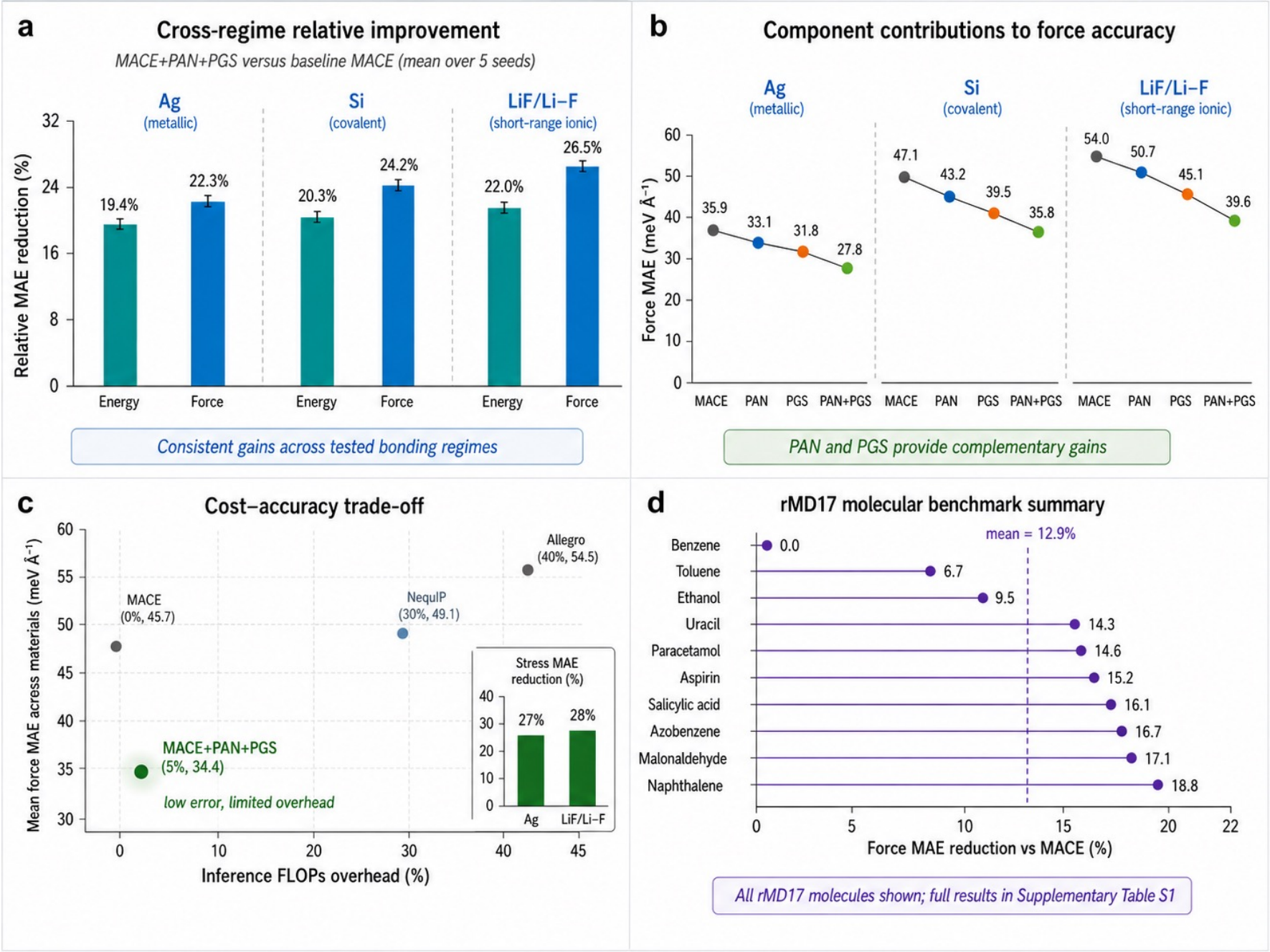}
    \caption{\textbf{PAN--PGS improves short-range accuracy across materials and molecular benchmarks with limited overhead.}
    \textbf{(a)} Relative reductions in energy and force MAE for MACE+PAN+PGS compared with baseline MACE across Ag, Si, and the LiF/Li--F short-range ionic subset. Bars show mean reductions over five training seeds; error bars indicate standard deviation. Absolute baseline force MAE values for the same systems are 35.9, 47.1, and 54.0~meV/\AA, respectively, reduced to 27.9, 35.8, and 39.6~meV/\AA\ after the PAN--PGS correction. \textbf{(b)} Component ablation for force MAE showing the effect of PAN, PGS, and their combination under the same MACE backbone. \textbf{(c)} Cost--accuracy trade-off using mean force MAE across Ag, Si, and LiF/Li--F versus inference FLOPs overhead. The inset shows the stress MAE in meV/\AA$^3$ on datasets with available stress labels (Ag and LiF/Li--F). \textbf{(d)} rMD17 molecular benchmark summary showing force-MAE reduction for MACE+PAN+PGS relative to baseline MACE across all ten rMD17 molecules. All comparisons in this figure are within-backbone, controlled in-house comparisons; external literature baselines are reported separately in Supplementary Table~\ref{tab:literature_ref} and follow different training and reporting protocols.}
    \label{fig:headline_results}
\end{figure*}

If scalar-pathway fidelity is a genuine design axis, it should yield measurable accuracy gains without relying on a larger tensor backbone or a substantially higher evaluation cost. We first asked whether targeted scalar-channel corrections translate into such gains across chemically distinct short-range materials. Across metallic Ag, covalent Si, and the short-range ionic LiF/Li--F subset, MACE+PAN+PGS systematically reduced both energy and force errors relative to baseline MACE (Fig.~\ref{fig:headline_results}(a); Table~\ref{tab:cross_system_mae_md17style}). Force-MAE reductions were 22.3\% for Ag, 24.0\% for Si, and 26.7\% for LiF/Li--F, with corresponding energy-MAE reductions of 19.4\%, 20.3\%, and 22.0\%. All comparisons were conducted under matched cutoffs, training budgets, and dataset splits, so that the scalar-pathway modification is the only intended architectural change within the shared short-range modeling protocol. Capacity expansion of the tensor backbone alone did not reproduce this effect: widened, deepened, and larger-radial-basis MACE controls reduced force MAE by only $4$--$7\%$, compared with $22$--$27\%$ for PAN+PGS at substantially lower parameter and FLOPs overhead ($1.03$--$1.05\times$ versus $1.4$--$1.7\times$ parameters and $1.1$--$1.5\times$ FLOPs; Supplementary Table~\ref{tab:capacity_matched_mace}). For the LiF/Li--F subset, a stricter Materials-Project-identifier-grouped split increased the absolute errors, as expected for the harder evaluation, but preserved the relative MACE+PAN+PGS force-MAE reduction at 26.5\% (Supplementary Table~\ref{tab:lif_grouped_split}), supporting the interpretation that the relative improvement is not a configuration-level split artifact.

\begin{table*}[t]
\caption{\textbf{Architecture-transfer test of scalar-pathway correction across
three short-range equivariant backbones.}
Energy and force MAEs are reported for Allegro, NequIP, and MACE before and after the insertion of architecturally analogous PAN+PGS scalar-pathway modifications. The controlled comparison is the within-backbone change from baseline to PAN+PGS; absolute errors across backbones are not interpreted as parameter-matched leaderboard rankings. Values are mean \(\pm\) standard deviation over five independent training seeds. Energies are in meV/atom, and forces are in meV/\AA. The LiF/Li--F entry corresponds to the short-range ionic subset filtered from MPtrj; grouped-split and computational-cost controls are reported in Supplementary Tables~\ref{tab:lif_grouped_split} and~\ref{tab:cost_comparison}.}
\label{tab:cross_system_mae_md17style}

\setlength{\tabcolsep}{0.8pt}
\renewcommand{\arraystretch}{0.8}
\arrayrulecolor{TableRule}

\resizebox{\textwidth}{!}{%
\begin{tabular}{@{}llcccccc@{}}
\toprule
\rowcolor{TableHead}
\textbf{System} & \textbf{Qty.}
& \textbf{Allegro}
& \shortstack{\textbf{Allegro}\\\textbf{+PAN+PGS}}
& \textbf{NequIP}
& \shortstack{\textbf{NequIP}\\\textbf{+PAN+PGS}}
& \textbf{MACE}
& \shortstack{\textbf{MACE}\\\textbf{+PAN+PGS}} \\
\midrule

\rowcolor{TableStripe}
Ag (metallic) & E
& $1.92 \pm 0.11$ & $1.59 \pm 0.09$
& $1.61 \pm 0.08$ & $1.32 \pm 0.07$
& $1.39 \pm 0.06$ & $\mathbf{1.12 \pm 0.05}$ \\

\rowcolor{TableStripe}
& F
& $41.2 \pm 1.6$  & $33.8 \pm 1.2$
& $37.4 \pm 1.2$  & $30.6 \pm 1.0$
& $35.9 \pm 0.9$  & $\mathbf{27.9 \pm 0.7}$ \\

\addlinespace[2pt]

Si (covalent) & E
& $3.51 \pm 0.17$ & $2.85 \pm 0.13$
& $2.83 \pm 0.13$ & $2.29 \pm 0.10$
& $2.41 \pm 0.10$ & $\mathbf{1.92 \pm 0.08}$ \\

& F
& $58.2 \pm 2.4$  & $47.5 \pm 1.7$
& $51.3 \pm 1.8$  & $41.8 \pm 1.4$
& $47.1 \pm 1.4$  & $\mathbf{35.8 \pm 1.0}$ \\

\addlinespace[2pt]

\rowcolor{TableStripe}
LiF/Li--F     & E
& $4.23 \pm 0.21$ & $3.45 \pm 0.16$
& $3.62 \pm 0.16$ & $2.93 \pm 0.13$
& $3.09 \pm 0.12$ & $\mathbf{2.41 \pm 0.10}$ \\

\rowcolor{TableStripe}
& F
& $64.1 \pm 2.7$  & $52.7 \pm 2.0$
& $58.6 \pm 2.1$  & $47.6 \pm 1.7$
& $54.0 \pm 1.6$  & $\mathbf{39.6 \pm 1.2}$ \\
\bottomrule
\end{tabular}%
}
\arrayrulecolor{black}
\end{table*}

Component ablations isolate the distinct contributions of the two modules (Fig.~\ref{fig:headline_results}(b)). PAN alone reduced force error across all three materials, consistent with improved scalar aggregation in geometrically heterogeneous local environments. PGS alone gave a larger reduction, consistent with enhanced representation of short-range PES curvature. The combined MACE+PAN+PGS achieved the lowest force MAE in every material: 27.9~meV/\AA\ for Ag, 35.8~meV/\AA\ for Si, and 39.6~meV/\AA\ for LiF/Li--F. This consistent ordering supports the interpretation that PAN and PGS act as complementary rather than interchangeable corrections.

We next asked whether the effect was tied to MACE or reflected a portable scalar-channel intervention. We therefore inserted PAN+PGS into Allegro and NequIP at architecturally analogous scalar-processing locations, using the same PAN descriptors, PGS basis sizes, datasets, splits, and five-seed protocol (Table~\ref{tab:cross_system_mae_md17style}). These experiments are not cross-backbone leaderboard comparisons; the controlled signal is the within-backbone change induced by the scalar-pathway correction. PAN+PGS reduced force MAE by 17--19\% in both Allegro and NequIP, compared with 22--27\% in MACE. We interpret this pattern as architecture-dependent portability: the sign of the effect is conserved across three distinct short-range equivariant backbones, whereas the magnitude is largest in MACE, where high-body-order tensor features impose the strongest compression demand on the scalar output pathway.

The accuracy gains come at a limited computational cost. MACE+PAN+PGS occupies the low-error, low-overhead region of the cost--accuracy frontier, adding approximately 5\% to inference FLOPs relative to baseline MACE (Fig.~\ref{fig:headline_results}(c)). On systems with available stress labels, MACE+PAN+PGS reduced stress MAE from $11.8 \pm 0.6$ to $8.6 \pm 0.4$~meV/\AA$^3$ for Ag (27\% reduction) and from $17.4 \pm 0.9$ to $12.6 \pm 0.6$~meV/\AA$^3$ for the LiF/Li--F subset (28\% reduction), where uncertainties are standard deviations over five independent training seeds. The full stress comparison across MACE-family variants is shown in the inset of Fig.~\ref{fig:headline_results}(c) and tabulated in Supplementary Table~\ref{tab:stress_mae}. Stress metrics are not reported for Si, as the public silicon dataset provides only energies and forces; the stress loss weight was accordingly set to zero for this system.

The same pattern extends to molecular benchmarks at shorter length scales. Across all ten rMD17 molecules, MACE+PAN+PGS reduced force MAE by an average of 13.9\% relative to baseline MACE (Fig.~\ref{fig:headline_results}(d)). Substantial gains were observed across several flexible and anharmonic molecules, while benzene showed negligible improvement. This hierarchy is physically interpretable: PAN and PGS deliver the greatest benefit when the scalar pathway must faithfully represent high-curvature, multi-well, or strongly anharmonic local structure--precisely the regimes where uniform aggregation and spectrally limited nonlinearities are most deficient. Where the baseline scalar representation is already adequate, the corrections are correspondingly modest.

\subsection{Geometric and spectral corrections operate via distinct scalar pathways}
\label{subsec:mechanism_results}

\begin{figure*}[!t]
    \centering
    \includegraphics[width=0.98\textwidth]{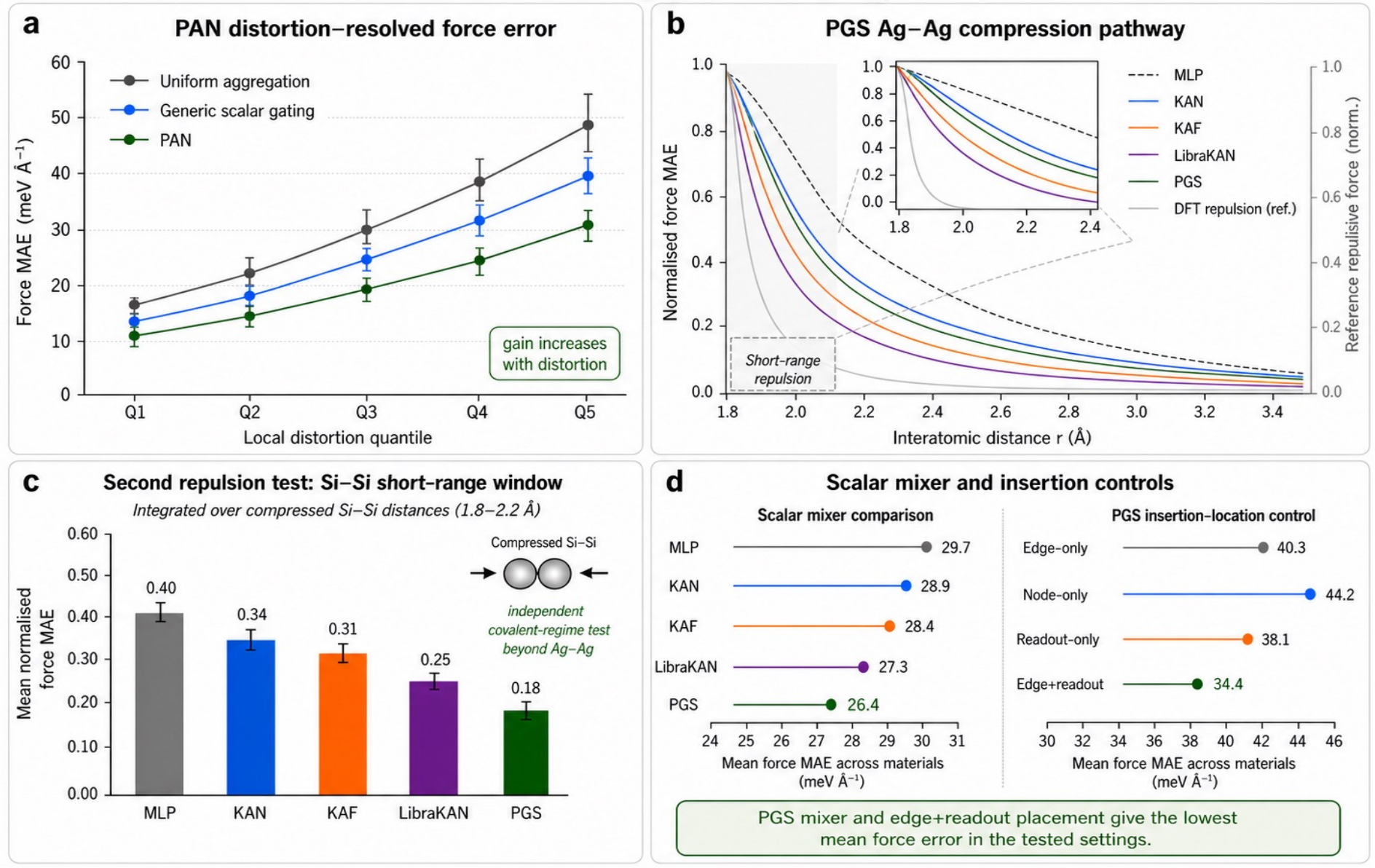}
    \caption{\textbf{PAN and PGS act through separable scalar mechanisms.}
    \textbf{(a)} Distortion-resolved force error for uniform aggregation, a generic scalar-gating control, and PAN. Error bars indicate standard deviation over five seeds. \textbf{(b)} Ag--Ag compression pathway showing normalized force MAE across scalar mixers, together with the reference repulsive-force profile. The shaded region marks the short-range repulsive regime. \textbf{(c)} Independent Si--Si short-range-window test. Values show mean normalized force MAE integrated over compressed Si--Si distances, with error bars indicating standard deviation over five seeds. \textbf{(d)} Scalar mixer and PGS insertion-location controls summarized as mean force MAE across Ag, Si, and the LiF/Li--F short-range ionic subset.}
    \label{fig:mechanism_results}
\end{figure*}

We next asked whether the improvements align with the intended mechanisms. PAN is designed to make scalar aggregation responsive to local geometric distortion. Force error increased from low- to high-distortion quantiles for all aggregation schemes, but PAN remained lowest across the range (Fig.~\ref{fig:mechanism_results}(a)). A generic scalar-gating control outperformed uniform aggregation, indicating that learned neighbor weighting is useful. PAN provided an additional reduction, especially in the most distorted local environments, showing that physically informed local descriptors add information beyond a generic gate. This trend also persisted under a stronger structural-class shift: when Ag surface or defect configurations were excluded from training and used only for testing, MACE+PAN+PGS retained force-MAE reductions of $34.8\%$ and $34.3\%$, respectively (Supplementary Table~\ref{tab:ag_heldout_classes}), supporting that the gain is not limited to random configuration-level interpolation but is most pronounced in the under-coordinated and defect-adjacent environments that motivate PAN.

PGS targets the complementary spectral limitation. Along an Ag--Ag compression pathway, all scalar mixers showed their largest errors in the short-range repulsive regime, where the force changes rapidly with distance (Fig.~\ref{fig:mechanism_results}(b)). PGS had the lowest normalized force error in this region and followed the reference repulsive-force profile more closely than the other tested scalar mixers. This supports the interpretation that spectral enrichment of scalar features improves the representation of steep short-range curvature.

To test whether this behavior is specific to Ag, we repeated the repulsion-focused analysis on compressed Si--Si distances (Fig.~\ref{fig:mechanism_results}(c)). The integrated normalized force error decreased from MLP to KAN, KAF, and LibraKAN, and was lowest for PGS. The Si--Si test, therefore, provides a covalent-regime counterpart to the Ag--Ag compression pathway, supporting the view that PGS improves short-range spectral resolution rather than fitting a material-specific artifact.

The control summaries further separate the effect of the mixer from that of the insertion location. Under matched PAN pooling, PGS gave the lowest mean force error among the tested scalar mixers (Fig.~\ref{fig:mechanism_results}(d), left). In the placement control, edge-only and readout-only PGS each improved over weaker placements, whereas the edge+readout configuration performed best (Fig.~\ref{fig:mechanism_results}(d), right). This result aligns with the architectural roles of the two PGS locations: edge-level PGS refines radial short-range information before tensor-product coupling. At the same time, readout-level PGS enriches the final invariant latent before energy prediction.

The controls in Fig.~\ref{fig:mechanism_results} make three simpler explanations of the observed gain less likely. First, the generic scalar-gating control in panel a indicates that PAN's improvement is not adequately explained by adding any learned gate, because the gain over the generic gate is recovered only when the gate uses physically informed local descriptors. Second, the MLP/KAN/KAF/LibraKAN comparison in panels (b)--(c) indicates that PGS's improvement is not adequately explained by simply increasing scalar nonlinearity capacity, because all mixers were matched in capacity and only the spectral form of the mixer correlated with short-range repulsion fidelity. Third, the placement control in panel (d) indicates that the gain is not adequately explained by inserting additional parameters near the readout, because both edge-only and readout-only placements were tested, and the edge+readout configuration outperformed each in isolation. The combined evidence supports a mechanistic, rather than purely capacity-based, interpretation of the scalar-pathway correction.

\subsection{Scalar-pathway fidelity transfers to physical observables}
\label{subsec:physical_fidelity}

\begin{figure*}[!t]
    \centering
    \includegraphics[width=0.98\textwidth]{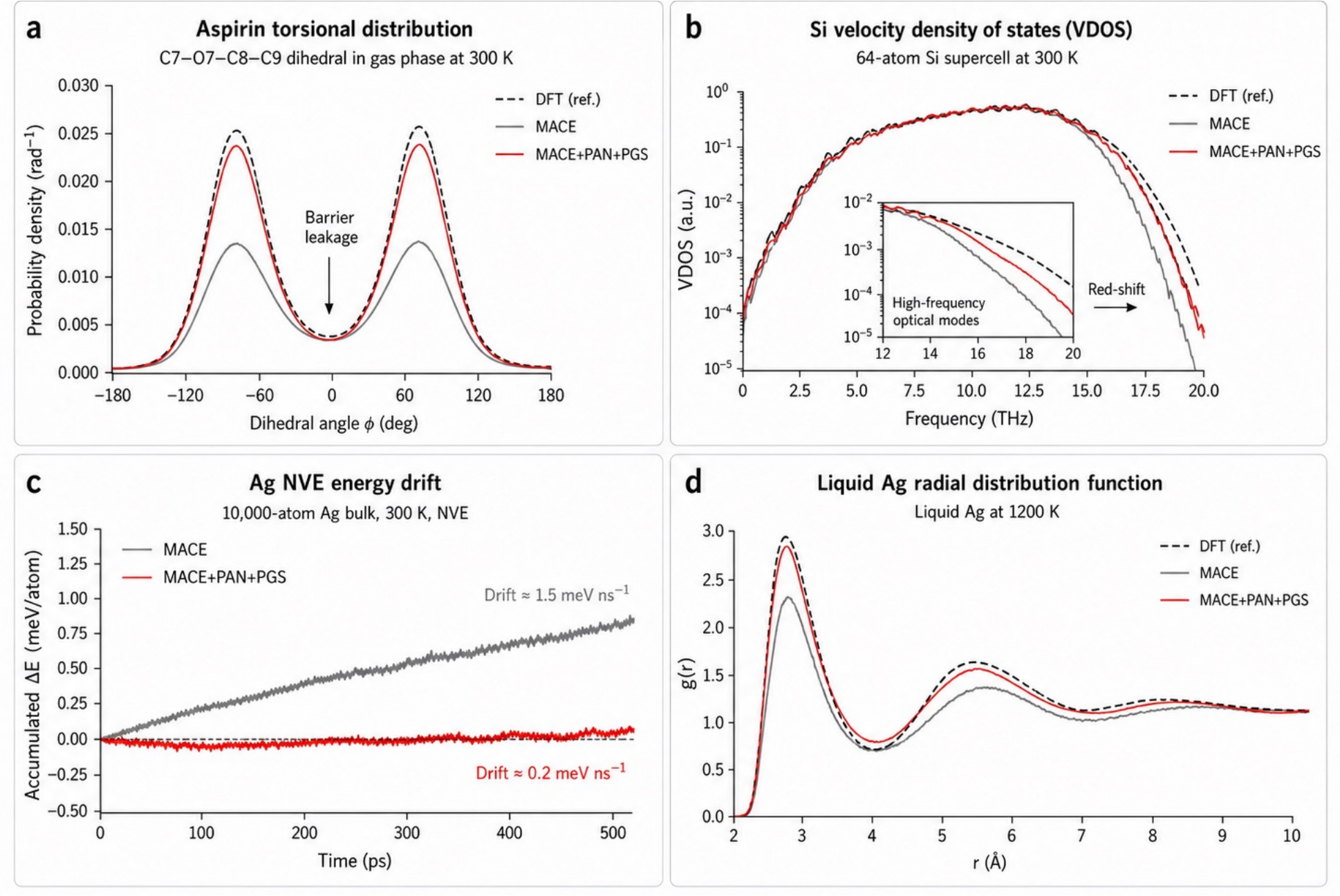}
    \caption{\textbf{Scalar-pathway correction improves observables controlled by stiff local forces.}
    \textbf{(a)} Aspirin C7--O7--C8--C9 torsional distribution at 300~K for DFT, baseline MACE and MACE+PAN+PGS. \textbf{(b)} Vibrational density of states of a 64-atom Si supercell at 300~K. The inset highlights the high-frequency optical-mode region. \textbf{(c)} Total-energy drift during a 500~ps NVE simulation of a 10{,}000-atom Ag bulk system at 300~K. Panel \textbf{(c)} shows a representative seed; aggregate drift across five seeds is reported in the main text. \textbf{(d)} Liquid Ag radial distribution function at 1{,}200~K.}
    \label{fig:physical_fidelity}
\end{figure*}

Static energy and force errors are necessary but insufficient evidence that a scalar-pathway bottleneck carries physical consequences. To close this gap, we selected four observables that correspond directly to the four failure modes predicted by the scalar-bottleneck hypothesis outlined in Section~\ref{subsec:framework}, rather than serving as generic accuracy checks. Each panel in Fig.~\ref{fig:physical_fidelity} maps one specific scalar-pathway deficiency to a measurable dynamical or structural quantity, and reports the distance to the DFT reference as mean $\pm$ standard deviation across the same five independent training seeds used for energy and force evaluation.

\medskip
\noindent\textbf{Spectral bias on anharmonic torsions.}
The first signature is barrier leakage in flexible anharmonic molecules---the dynamical fingerprint of low-frequency spectral bias in the scalar pathway. A spectrally limited scalar mixer can over-smooth the steep PES curvature near torsional barriers, redistributing probability density into low-probability barrier regions. In aspirin at 300~K, baseline MACE produced a markedly flattened C7--O7--C8--C9 torsional distribution with anomalous excess probability near the barrier top (Fig.~\ref{fig:physical_fidelity}(a)). The Kullback--Leibler divergence $D_{\mathrm{KL}}(p_{\mathrm{model}}\|p_{\mathrm{DFT}})$ decreased from $0.31 \pm 0.04$ nats for baseline MACE to $0.09 \pm 0.02$ nats for MACE+PAN+PGS---a roughly threefold reduction concentrated in the barrier region, consistent with improved representation of high-curvature torsional structure in the scalar readout.

\medskip
\noindent\textbf{Spectral bias on stiff vibrational modes.}
The second signature is a systematic red-shift of high-frequency optical phonons, arising when smooth scalar nonlinearities under-resolve the curvature of stiff covalent bonds. In bulk Si, baseline MACE displaced the dominant optical band in the vibrational density of states (VDOS) to anomalously low frequencies (Fig.~\ref{fig:physical_fidelity}(b)). MACE+PAN+PGS reduced the mean absolute peak-position error from $1.42\pm0.12$~THz to $0.31\pm0.07$~THz, and reduced the integrated $L_1$ error over the optical window ($14$--$18$~THz) by a factor of $4.1$---from $0.094\pm0.008$ to $0.023\pm0.004$ in normalized units. This fourfold improvement in VDOS fidelity is consistent with the PGS spectral enrichment, which improves the representation of short-range curvature information that smooth nonlinearities suppress.

\medskip \noindent\textbf{Finite-step energy drift in long-time integration.}
The third signature is secular energy drift under finite-timestep microcanonical integration. We emphasize that forces are obtained by automatic differentiation of the scalar energy (Eq.~\ref{eq:forces_stress}), so the force field is conservative by construction; residual drift therefore reflects finite-timestep integration error coupled to scalar-pathway curvature errors, not non-conservative force components. Scalar under-resolution inflates effective PES stiffness, amplifying this drift even when short-horizon force errors appear acceptable. In a 500 ps NVE simulation of a 10{,}000-atom bulk Ag system, baseline MACE accumulated a positive drift of $1.48\pm0.18$~meV~ns$^{-1}$ per atom; MACE+PAN+PGS reduced this to $0.21\pm0.04$~meV~ns$^{-1}$ per atom under identical protocol (Fig.~\ref{fig:physical_fidelity}(c))---a sevenfold reduction. A paired two-sided $t$-test on per-seed drift slopes yields $p<10^{-3}$, supporting that the improvement is statistically robust across training seeds.

\medskip \noindent\textbf{Geometric under-resolution on liquid structure.}
The fourth signature is anomalous nearest-neighbor over-approach in liquid structure---the dynamical analog of the static short-range repulsion errors quantified in Fig.~\ref{fig:mechanism_results}(b), consistent with softened short-range repulsive walls under insufficiently adaptive scalar aggregation. In liquid Ag at 1{,}200~K, baseline MACE generated a spurious sub-shell peak in $g(r)$ at $r\approx 2.3$~\AA, corresponding to unphysical atomic over-approach (Fig.~\ref{fig:physical_fidelity}(d)). The integrated $L_1$ distance $\int_{0}^{6.0\,\text{\AA}} |g_{\mathrm{model}}(r)-g_{\mathrm{DFT}}(r)|\,\mathrm{d}r$ decreased from $0.181\pm0.012$ for baseline MACE to $0.061\pm0.008$ for MACE+PAN+PGS---a threefold reduction---and the spurious sub-shell peak was substantially attenuated.

Across the four observables, the same scalar-pathway correction attenuates four otherwise distinct physical signatures across molecular, crystalline, and liquid settings. We interpret this convergence as mechanism-level evidence rather than additional benchmark coverage: the bottleneck identified in Section~\ref{subsec:framework} and the mechanistic controls in Section~\ref{subsec:mechanism_results} make specific physical predictions, and Fig.~\ref{fig:physical_fidelity} shows that those predictions are borne out in MD-derived observables and are quantitatively attenuated by the corresponding correction.

\section{Discussion}
This work identifies the invariant scalar pathway as a controllable interface for the physical accuracy of short-range equivariant interatomic potentials. In MACE and related $O(3)$-equivariant architectures, tensor-product message passing constructs rich geometric representations, but the final energy and its derivatives are recovered through invariant scalar pathways. Our results provide evidence that this conversion stage is not an implementation detail: how scalar information is aggregated and spectrally mixed measurably constrains the accuracy of energy, force, and stress in the resulting potential and is associated with specific, predictable signatures in molecular-dynamics observables. Tensor expressivity and scalar fidelity are therefore not substitutable design axes; both contribute to short-range physical accuracy.

The PAN and PGS modules address two distinct forms of scalar information loss. PAN targets the geometric loss introduced by uniform neighbor aggregation. A permutation-invariant sum is symmetry preserving, but it does not by itself distinguish coordination loss, density variation, or local anisotropy. By using invariant local descriptors to gate scalar aggregation, PAN improves accuracy, particularly in distorted environments, while retaining an unchanged equivariant tensor-product backbone. This result suggests that part of the error often attributed to limited model capacity can instead arise from the compression of local geometric information into scalar channels.

PGS addresses the complementary spectral limitation. Smooth scalar nonlinearities tend to represent low-frequency variation more readily than steep local curvature. In atomistic potentials, this bias is most visible in short-range repulsion, high-frequency vibrations, and anharmonic molecular coordinates. The Ag--Ag and Si--Si compression tests show that PGS improves the local representation of repulsive curvature beyond generic scalar mixers. The same mechanism is reflected at the simulation level: MACE+PAN+PGS better preserves the high-frequency tail of the Si vibrational density of states and reduces barrier leakage in aspirin torsional sampling. These observations connect a local scalar spectral correction to physical quantities that depend on stiff short-range forces.

This separation between scalar and tensor pathways is important for interpreting the gains. The improvement is not evidence that equivariant representations are unnecessary, nor that scalar processing alone fully accounts for performance. It shows that an expressive equivariant backbone can still lose useful physical detail when tensorial information is processed by insufficiently adaptive or spectrally limited scalar operations downstream.

We use the labels ``physics-aware'' (PAN) and ``physics-guided'' (PGS) in the sense of symmetry-preserving descriptor and basis design: the gating descriptors and spectral bases are chosen to mirror physically interpretable scalar quantities (local coordination, angular variance, short-range spectral content) rather than being treated as free black-box operators. PAN--PGS does not impose hard physical constraints such as exact asymptotic repulsion, charge conservation, or thermodynamic consistency, and the ``physical fidelity'' improvements reported here refer to the specific quantitative observables in Section~\ref{subsec:physical_fidelity} (torsional KL divergence, VDOS peak position, NVE drift, RDF $L_1$ distance) rather than to global physical correctness.

PAN--PGS defines a short-range scalar-interface correction rather than a replacement for explicit physical terms. This operating regime follows directly from where the modules act: after the equivariant local representation has been constructed, but before invariant scalar features are converted into atomic energies and their derivatives. The correction can therefore improve the expression of local coordination, short-range curvature, and stiff modes through the energy pathway. However, it cannot introduce physical interactions that are absent from the underlying short-range model. Long-range Coulomb interactions, charge transfer, polarizability, dispersion, and reactive bond rearrangements still require explicit physical treatment. In hybrid potentials, PAN--PGS would accordingly complement such terms by reducing the residual short-range scalar error, while electrostatic, charge-response, dispersion, or reactive components handle their respective physics. The LiF/Li--F experiments should be read in this sense: they probe short-range ionic local environments under held-out configurations from MPtrj, not full long-range Coulomb generalization.

This distinction also defines the contribution. The results do not suggest that scalar processing can replace equivariance or physical modeling; they show that a strong equivariant local representation can still lose accuracy at the scalar interface, where it becomes an energy landscape. Scalar aggregation and scalar spectral mixing are therefore not implementation details, but design axes for short-range equivariant MLIPs. In this view, equivariance determines what geometric information is available, whereas scalar-pathway fidelity determines how much of that information survives the conversion into a deployable energy surface.

The same logic predicts where PAN--PGS will not help. Systems whose short-range scalar pathway is already well-resolved relative to the training data---near-harmonic small molecules such as benzene in our per-molecule breakdown (Supplementary Table~\ref{tab:rmd17_md17_full})---show essentially no gain, as expected. Systems where the binding error is not in the scalar pathway---reactive bond breaking that moves the equivariant feature distribution outside the trained range, or systems dominated by long-range physics that the backbone cannot represent---will not benefit from a scalar-pathway correction alone. These are not weaknesses of the method but consequences of the principle: scalar-pathway fidelity is a binding constraint only when the scalar pathway is what binds.

These limitations also point to the next direction. Recent progress in equivariant interatomic potentials has focused on increasing tensor expressivity through higher-body order, deeper message passing, and richer equivariant features. Our results suggest that such gains can be partly masked when the scalar pathway remains a low-resolution interface, and that treating scalar aggregation and scalar spectral mixing as first-class design components provides a complementary route to short-range physical accuracy. Rather than expanding only the equivariant backbone, future short-range MLIPs should also strengthen the scalar channel through which equivariant information is converted into energies and derivatives. Both stages must be designed explicitly.

\section{Methods}
\label{sec_meth}

\subsection{Scalar-pathway architecture}

\medskip\noindent\textbf{Backbone and notation.}
All MACE-family variants used the same short-range MACE backbone~\cite{batatia2022mace}: cutoff radius \(r_{\mathrm{cut}}=5.0\)~\AA, two interaction blocks, correlation order \(\nu=3\), \(N_{\mathrm{rad}}=8\) Bessel radial functions, and hidden irreps \(128\times0e+128\times1o+128\times2e\) followed by a \(128\times0e\) readout. These settings were fixed for baseline MACE, MACE+PAN, MACE+PGS and MACE+PAN+PGS. For an edge \((i,j)\), the baseline embedding is
\begin{equation}
    \phi^{(\ell,m)}_{ij,n}
    =
    B_n(r_{ij})Y_\ell^m(\hat{\mathbf r}_{ij}),
    \qquad n=1,\ldots,N_{\mathrm{rad}},
\end{equation}
where \(r_{ij}=\|\mathbf r_j-\mathbf r_i\|\), \(\hat{\mathbf r}_{ij}=(\mathbf r_j-\mathbf r_i)/r_{ij}\), \(B_n\) is the cutoff-smoothed radial basis and \(Y_\ell^m\) is a real spherical harmonic. We denote the scalar edge channel by \(h_{ij}^{(0)}\in\mathbb R^{C_e}\) and the final invariant node descriptor by \(x_i^{(0)}\in\mathbb R^{C_n}\), with \(C_e=C_n=128\). Atomic energies were predicted from the final scalar descriptor \(y_i^{(0)}\),
\begin{equation}
    E_{\mathrm{tot}}=\sum_i \varepsilon_\theta(y_i^{(0)}),
\end{equation}
and forces and stress were obtained as analytic derivatives of the same scalar energy,
\begin{equation}
    \mathbf F_i
    =
    -\frac{\partial E_{\mathrm{tot}}}{\partial \mathbf r_i},
    \qquad
    \boldsymbol\sigma
    =
    \frac{1}{V}\frac{\partial E_{\mathrm{tot}}}{\partial \boldsymbol\epsilon}.
    \label{eq:forces_stress}
\end{equation}
Here \(V\) is the cell volume and \(\boldsymbol\epsilon\) is the symmetric cell strain. Stress labels and predictions follow the \textsc{VASP} sign and volume convention and are reported in meV/\AA\(^3\). All \(\ell>0\) tensor-product operations of the MACE backbone were left unchanged.

\medskip\noindent\textbf{Physics-aware neighborhood pooling.}
PAN replaces uniform scalar aggregation with sigmoid-gated scalar aggregation. For each directed edge \((i,j)\), we formed the invariant descriptor
\[
    \psi_{ij}
    =
    \big[
    h_{ij}^{(0)},\,
    r_{ij},\,
    Z_i,\,
    Z_j,\,
    \rho_i,\,
    \eta_i
    \big],
\]
where \(Z_i\) and \(Z_j\) are atomic numbers, \(\rho_i\) is a smooth coordination descriptor and \(\eta_i\) is a cutoff-weighted cosine-angle variance:
\begin{align}
    \rho_i  &= \sum_{j\in\mathcal N(i)} f_{\mathrm{cut}}(r_{ij}), \\
    c_{jik} &= \hat{\mathbf r}_{ij}\cdot \hat{\mathbf r}_{ik}, \\
    \bar c_i
    &=
    \frac{
        \sum_{j,k \in \mathcal N(i),\, j\neq k}
        f_{\mathrm{cut}}(r_{ij}) f_{\mathrm{cut}}(r_{ik}) c_{jik}
    }{
        \epsilon_{\mathrm{reg}}
        + \sum_{j,k \in \mathcal N(i),\, j\neq k}
        f_{\mathrm{cut}}(r_{ij}) f_{\mathrm{cut}}(r_{ik})
    }, \\
    \eta_i  &= \frac{
        \sum_{j,k \in \mathcal N(i),\, j\neq k}
        f_{\mathrm{cut}}(r_{ij}) f_{\mathrm{cut}}(r_{ik})
        \big(c_{jik} - \bar c_i\big)^2
    }{
        \epsilon_{\mathrm{reg}}
        + \sum_{j,k \in \mathcal N(i),\, j\neq k}
        f_{\mathrm{cut}}(r_{ij}) f_{\mathrm{cut}}(r_{ik})
    }.
\end{align}
We used \(\epsilon_{\mathrm{reg}}=10^{-6}\) for numerical regularization. The use of cosine angles avoids an \(\arccos\) operation, and all descriptors are invariant under rotations, translations and neighbor reordering.

The PAN gate and scalar aggregation are
\begin{equation}
    a_{ij}
    =
    \sigma\!\left(f_\theta(\psi_{ij})\right),
    \qquad
    h_i^{(0)}=\sum_j a_{ij}\hat h_{ij}^{(0)},
    \qquad
    h_i^{(0)}
    =
    \sum_{j\in\mathcal N(i)}
    \tilde h_{ij}^{(0)} .
    \label{eq:pan_pool}
\end{equation}
When edge-level PGS is used, \(h_{ij}^{(0)}\) in Eq.~\eqref{eq:pan_pool} is replaced by the PGS-modified scalar edge descriptor \(\hat h_{ij}^{(0)}\). The gate is not softmax-normalized, so the total scalar amplitude can vary with local coordination. In the generic scalar-gating control, the same gate architecture was used but \(\rho_i\) and \(\eta_i\) were removed from \(\psi_{ij}\).

\medskip\noindent\textbf{Physics-guided spectral mixers.}
PGS enriches scalar spectral resolution at the edge and readout levels. Edge-level PGS modifies the scalar edge descriptor before tensor-product coupling:
\begin{equation}
    \hat h_{ij}^{(0)}
    =
    h_{\mathrm{low}}\!\left(h_{ij}^{(0)},r_{ij}\right)
    +
    g_\theta\!\left(h_{ij}^{(0)}\right)
    \left[
    \sum_{m=1}^{M}
    \alpha_m^{\mathrm{edge}} B_m^{\mathrm{FB}}(r_{ij})
    \right],
    \label{eq:edge_pgs}
\end{equation}
where \(B_m^{\mathrm{FB}}\) are real Fourier--Bessel radial basis functions multiplied by the same smooth cutoff envelope as the MACE radial embedding. Readout-level PGS acts on the final scalar node descriptor:
\begin{equation}
    y_i^{(0)}
    =
    h_{\mathrm{low}}\!\left(x_i^{(0)}\right)
    +
    \sum_{m=1}^{M}
    \boldsymbol\alpha_m^{\mathrm{readout}}\,
    K_m^{\mathrm{ES}}\!\left(u_m^\top x_i^{(0)}\right),
    \label{eq:readout_pgs}
\end{equation}
where \(u_m\in\mathbb R^{C_n}\) is a learned projection direction, \(K_m^{\mathrm{ES}}\) is a tapered Exponential-of-Semicircle basis function, and \(\boldsymbol\alpha_m^{\mathrm{readout}}\in\mathbb R^{C_n}\) lifts the scalar basis response back to the node-feature dimension. The explicit cutoff envelope, Fourier--Bessel basis, ES kernel, descriptor normalization and forward-pass pseudocode are given in Supplementary Section~\ref{sec:si_algorithmic_details}.

\medskip\noindent\textbf{Symmetry and differentiability.}
\begin{proposition}[Symmetry and differentiability preservation]
\label{prop:symmetry_diff}
Applying PAN and PGS as defined in Eqs.~\eqref{eq:pan_pool}--\eqref{eq:readout_pgs} preserves (i) permutation invariance of the energy, (ii) \(O(3)\)-equivariance of the predicted forces and stress tensor, and (iii) the differentiability required for analytic force and stress evaluation on noncoincident atomic configurations under the fixed-candidate-neighbor convention and the smooth cutoff/tapered-kernel implementation described in Supplementary Section~\ref{sec:si_algorithmic_details}.
\end{proposition}
\noindent PAN and PGS modify only \(\ell=0\) channels and leave all \(\ell>0\) tensor-product operations unchanged. The detailed proof, including neighbor-list convention boundaries, is provided in Supplementary Section~\ref{sec:si_proposition_proof}.

\subsection{Datasets, splits and training}

\medskip\noindent\textbf{Datasets and splits.}
The study used four data sources: author-generated Ag data, public Si data, an MPtrj-derived LiF/Li--F subset, and public molecular benchmarks.

\medskip\noindent\textbf{Ag.}
Metallic Ag reference data were generated by the authors using DFT with the PBE functional and PAW pseudopotentials in \textsc{VASP}. The dataset includes bulk fcc, strained, defective, surface, and high-temperature configurations, with stress tensors stored for all configurations. The final Ag dataset contains \(4.6\times10^4\) configurations; duplicate configurations were removed before an 80/10/10 train/validation/test split. Structural-class held-out tests for surface and defect configurations are reported separately in Supplementary Table~\ref{tab:ag_heldout_classes}. Full DFT convergence parameters, smearing settings, PAW identifiers, and \(k\)-point sampling rules are listed in Supplementary Table~\ref{tab:dft_settings}.

\medskip\noindent\textbf{Si.}
The Si experiments used the public \texttt{Silicon\_MLIP\_datasets} repository. Released structures and DFT energy--force labels were converted to the common MACE-compatible data format used for all MACE-family variants. The release used here does not provide stress labels, so the stress-loss weight was set to zero for Si, and stress accuracy was not reported. The Si subset contains \(3.4\times10^4\) configurations. Release-provided splits were used when available; otherwise, configurations were split 80/10/10 with a fixed seed.

\medskip\noindent\textbf{LiF/Li--F.}
The short-range ionic subset was filtered from the Materials Project Trajectory dataset (MPtrj)~\cite{deng2023chgnet}, which contains DFT energies, forces, and stresses parsed from Materials Project GGA/GGA+U \textsc{VASP} static and relaxation trajectories~\cite{jain2013commentary}. We retained configurations whose formula contains both Li and F, including stoichiometric LiF and other Li--F-containing chemistries. We removed configurations that were more than 1.5 eV/atom above the corresponding ground-state structure. The retained subset contains \(3.9\times10^4\) configurations and was split 80/10/10 at the configuration level. A stricter Materials-Project-identifier-grouped split, used to test whether the relative improvement survives trajectory-level grouping, is reported in Supplementary Section~\ref{sec:si_grouped_split}.

\medskip\noindent\textbf{Molecular benchmarks.}
The molecular experiments used the public MD17 and revised MD17 datasets~\cite{Chmiela2017MD17, Christensen2020rmd17}. Training budgets were \(N_{\mathrm{train}}=1{,}000\) for rMD17 and \(N_{\mathrm{train}}=950\) for original MD17; validation and test partitions followed the split files used in the benchmark scripts. MD17 and rMD17 do not provide stress labels, so the stress loss weight was set to zero.

\medskip\noindent\textbf{Training protocol and model matching.}
All MACE-family variants were trained with the \textsc{ASE}--\textsc{MACE} framework on NVIDIA A100 GPUs in float64 precision. Training used AdamW with an initial learning rate of \(8\times10^{-3}\), weight decay \(10^{-8}\), ReduceLROnPlateau scheduling (patience 5, factor 0.5), exponential moving average decay 0.99999, stochastic weight averaging after epoch 30, batch size 4, and gradient clipping at 100. The training objective was
\begin{equation}
	\mathcal L
	=
	w_E |E_{\mathrm{pred}}-E_{\mathrm{DFT}}|
	+
	w_F \|\mathbf F_{\mathrm{pred}}-\mathbf F_{\mathrm{DFT}}\|_1
	+
	w_\sigma \|\boldsymbol\sigma_{\mathrm{pred}}-\boldsymbol\sigma_{\mathrm{DFT}}\|_1 ,
	\label{eq:loss}
\end{equation}
with \(w_E=1.0\), \(w_F=10.0\), and \(w_\sigma=10.0\) when stress labels were available; otherwise \(w_\sigma=0\). Across baseline MACE, MACE+PAN, MACE+PGS, and MACE+PAN+PGS, the cutoff, radial basis, irreps layout, interaction depth, optimizer settings, data splits, and training budget were held fixed. PAN used a one-hidden-layer MLP of width 64, and both edge-level and readout-level PGS used \(M=16\) spectral components.

\subsection{Cross-backbone and ablation controls}
\label{sec:methods_cross_backbone}

\medskip\noindent\textbf{Architecture-transfer controls.}
The cross-backbone entries in Table~\ref{tab:cross_system_mae_md17style} were designed as architecture-transfer controls. We inserted the same PAN descriptors and PGS basis sizes into Allegro~\cite{musaelian2023learning} and NequIP~\cite{batzner2022e3} at functionally analogous scalar-processing locations. For each backbone, the controlled comparison is between the baseline model and its PAN+PGS variant, across matched datasets, splits, training budgets, and the five-seed protocol. All \(\ell>0\) tensor-product operations and backbone-specific equivariant blocks were left unchanged. These controls test whether the scalar-pathway correction remains beneficial when the surrounding equivariant architecture changes; they are not intended as parameter-matched comparisons of absolute MAE across different backbones. The detailed insertion mapping, backbone hyperparameters, and capacity/FLOPs matching policy are provided in Supplementary Section~\ref{sec:si_cross_backbone} and Supplementary Table~\ref{tab:cross_backbone_insertion}.

\medskip\noindent\textbf{Capacity-matched controls.}
A capacity-matched MACE control, in which the baseline MACE backbone is widened, deepened, or given a larger radial basis under a matched training budget, is reported in Supplementary Table~\ref{tab:capacity_matched_mace}. The capacity-matched variants achieve only \(4\)--\(7\%\) force-MAE reductions relative to baseline MACE---substantially smaller than the \(22\)--\(27\%\) reductions reported here for MACE+PAN+PGS---and require \(1.4\)--\(1.7\times\) more parameters and \(1.1\)--\(1.5\times\) more inference FLOPs. This supports the interpretation that the gain in Table~\ref{tab:cross_system_mae_md17style} is not adequately explained by backbone capacity.

\subsection{Evaluation protocols and statistical analysis}

\medskip\noindent\textbf{Statistical reporting.}
All main accuracy results were computed over five independent training seeds \(\{1,2,3,4,5\}\) with identical data splits, hyperparameters, and training budgets. Values in the main figures and tables are reported as mean \(\pm\) standard deviation unless stated otherwise. Energy, force, and stress MAE values are reported under the same five-seed protocol; stress MAE is reported only for systems with available stress labels (Ag and LiF/Li--F). Error bars in Fig.~\ref{fig:headline_results}, Fig.~\ref{fig:mechanism_results}(a) and Fig.~\ref{fig:mechanism_results}(c) denote standard deviation over the five training seeds. The large-scale Ag NVE drift statistics were computed over the same five trained models; the trajectory shown in Fig.~\ref{fig:physical_fidelity}(c) is a representative seed.

For the physical-observable metrics in Section~\ref{subsec:physical_fidelity}, we quantified agreement with the DFT or DFT-MD reference using the Kullback--Leibler divergence for aspirin torsional distributions, peak-position and optical-window \(L_1\) errors for Si VDOS, integrated \(L_1\) distance for liquid Ag RDF, and linear drift slopes for Ag NVE trajectories. NVE drift significance was assessed by a paired two-sided \(t\)-test over the five training seeds. Histogram definitions, smoothing constants, fitting windows, bin widths, and reference-trajectory sampling details are provided in Supplementary Section~\ref{sec:si_md_protocols}.

\medskip\noindent\textbf{Mechanistic diagnostics.}
For distortion-resolved force errors in Fig.~\ref{fig:mechanism_results}(a), each local environment was assigned a distortion score from standardized coordination deficit, cosine-angle variance, and local-density deviation, then binned into five equal-population quantiles. For the Ag--Ag compression pathway in Fig.~\ref{fig:mechanism_results}(b), a top-layer Ag pair was displaced along the surface-normal compression coordinate, and force errors were normalized by the maximum reference repulsive force. For the Si--Si short-range window in Fig.~\ref{fig:mechanism_results}(c), compressed diamond-Si configurations were evaluated over \(1.8\le r_{\mathrm{Si-Si}}\le2.2\)~\AA. Scalar-mixer controls replaced only the scalar mixer under fixed PAN and MACE settings, and insertion-location controls placed PGS at the edge, node, readout, or edge+readout positions. Fig.~\ref{fig:mechanism_results}(d) reports mean force MAE averaged over Ag, Si, and LiF/Li--F.

\medskip\noindent\textbf{Molecular-dynamics observables and DFT reference quality.}
The four physical observables in Fig.~\ref{fig:physical_fidelity} were computed under matched thermostat, timestep, equilibration, and production conventions for baseline MACE and MACE+PAN+PGS, with DFT or DFT-MD references generated under the corresponding conditions. Because the DFT-MD reference trajectories used for VDOS and RDF estimation are necessarily shorter than the corresponding MLIP trajectories, the reported curve distances should be interpreted relative to finite-sampling DFT reference curves. Reference trajectory lengths, sampling intervals, and finite-sampling uncertainty considerations are reported in Supplementary Section~\ref{sec:si_md_protocols}.

\subsection{Computational cost and reproducibility}

\medskip\noindent\textbf{Computational cost.}
Relative computational cost was measured under matched backbone, cutoff, batching strategy, hardware, and force/stress evaluation protocol. Inference FLOPs include the automatic differentiation backward pass through the scalar energy with respect to atomic positions and cell strain. Because PAN and PGS operate only on scalar channels and do not alter tensor-product operations, their measured inference-FLOPs overhead is approximately \(5\%\) relative to baseline MACE. Detailed parameter counts, training-time ratios, inference FLOPs, wall-clock molecular-dynamics throughput, and capacity-matched runtime controls are reported in Supplementary Table~\ref{tab:cost_comparison} and Supplementary Section~\ref{sec:si_cost}.

\medskip\noindent\textbf{Code and reproducibility.}
All dataset filters, split files, model configurations, training scripts, analysis scripts, molecular-dynamics input decks, and figure-source data are provided at \href{https://github.com/JIABI/mace_librakan}{\nolinkurl{github.com/JIABI/mace_librakan}} and archived on Zenodo as stated in the Data and Code Availability sections. The release includes source data for all main and supplementary figures and tables, including compression paths, distortion-bin assignments, RDF/VDOS post-processing scripts, and NVE drift fits. Selected experiments were re-run on A100 hardware to verify numerical reproducibility.

\section{Data availability}
The datasets analyzed in this study comprise one author-generated dataset and three public datasets. The authors generated the Ag reference data. The Si experiments use the public \texttt{Silicon\_MLIP\_datasets} repository; the LiF/Li--F short-range ionic subset is filtered from the public Materials Project Trajectory dataset (MPtrj); and the molecular benchmark experiments use the public MD17 and rMD17 datasets.

The author-generated Ag DFT reference data---including bulk, strained, defective, surface, high-temperature, liquid, and short-range compression configurations---were produced by the authors using \textsc{VASP}. POTCAR pseudopotential files from the \textsc{VASP} distribution are not redistributed, in line with the \textsc{VASP} license; all other raw outputs (OUTCAR, vasprun.xml, OSZICAR) for representative configurations are deposited together with the processed dataset. The processed Ag data required to reproduce the reported Ag results---including structures, per-configuration energies and forces, available stress labels, neighbor lists, curated train/validation/test splits, RDF and NVE trajectory subsets, and the Ag--Ag compression-pathway source data---are provided in the accompanying data/code release. Any additional processed derivatives or metadata that exceed deposit size limits are available from the corresponding author upon reasonable request, subject to institutional data-sharing and third-party license constraints.

The Si experiments use the public \texttt{Silicon\_MLIP\_datasets} repository at \href{https://github.com/mogroupumd/Silicon_MLIP_datasets}{\nolinkurl{github.com/mogroupumd/Silicon_MLIP_datasets}}. The released structures and DFT energy and force labels were converted to the common training and evaluation format used by all MACE-family variants in this work. The processed Si splits, conversion scripts, and Si--Si short-range-window source data used in the mechanistic analysis are included in the accompanying code/data release.

The LiF/Li--F short-range ionic tests use public inorganic trajectory data filtered from the Materials Project Trajectory dataset (MPtrj), available at \href{https://figshare.com/articles/dataset/23713842}{\nolinkurl{https://figshare.com/articles/dataset/23713842}}. MPtrj contains structures with DFT energies, forces, and stresses parsed from Materials Project GGA/GGA+U \textsc{VASP} static and relaxation trajectories. In this study, LiF and Li--F-containing configurations were selected to construct a short-range ionic evaluation subset. The exact material identifiers, filtering criteria, processed structures, labels, and train/validation/test split files used for this subset are included in the accompanying code/data release.

The molecular benchmark experiments use the public MD17 and revised MD17 (rMD17) datasets. MD17 is available from the Quantum Machine dataset portal at \href{https://quantum-machine.org/datasets/}{\nolinkurl{quantum-machine.org/datasets/}}. The revised MD17 dataset is available from Figshare at \href{https://figshare.com/articles/dataset/Revised_MD17_dataset_rMD17_/12672038}{\nolinkurl{figshare.com/articles/dataset/Revised_MD17_dataset_rMD17_/12672038}}.

All derived numerical data supporting the findings of this manuscript---including source data for all main and supplementary figures, processed evaluation metrics, ablation summaries, distortion-bin assignments, compression-path source data, VDOS and RDF source curves, NVE drift fits, and molecular-dynamics-derived observables---are provided in the accompanying code/data release through the public GitHub repository linked above and have been archived on Zenodo at \href{https://doi.org/10.5281/zenodo.20181865}{\nolinkurl{10.5281/zenodo.20181865}}. The repository and the archival snapshot contain processed Ag, Si, and LiF/Li--F datasets, configuration files, filter scripts, split files, and figure-source data exactly as used to produce the results reported here, together with the corresponding software environment manifest (\texttt{environment.yml} and pinned package versions).

\section{Code availability}
All source code required to train and evaluate the models, including the implementations of PAN and PGS, is available at \href{https://github.com/JIABI/mace_librakan}{\nolinkurl{github.com/JIABI/mace_librakan}}. The exact code version used to produce the results reported in this manuscript has been archived on Zenodo at \href{https://doi.org/10.5281/zenodo.20181865}{\nolinkurl{10.5281/zenodo.20181865}}, which serves as the persistent reference snapshot for this work.

The repository (and the corresponding Zenodo snapshot) contains: (i) the MACE-compatible model definitions for PAN, edge-level PGS, and readout-level PGS, including the cross-backbone insertions used for Allegro and NequIP; (ii) training and evaluation scripts and configuration files for all reported experiments, including the five-seed protocol and the capacity-matched MACE controls; (iii) molecular-dynamics analysis scripts and input decks for the aspirin torsion, Si VDOS, Ag NVE, and liquid Ag RDF observables; (iv) filtering scripts and split files used to construct the LiF/Li--F subset from MPtrj, the conversion scripts for the public \texttt{Silicon\_MLIP\_datasets} release, and the post-processing scripts for the Ag--Ag and Si--Si compression analyses, VDOS, RDF, distortion-bin analysis, and NVE drift fits; and (v) all figure-generation scripts together with the processed source data underlying the main and supplementary figures and tables.

A reproducibility guide is provided in the repository, including environment specifications (\texttt{environment.yml} with pinned PyTorch, CUDA, e3nn, and ASE versions), example command-line invocations, and step-by-step instructions for reproducing the main tables, figures, and statistical summaries. The Zenodo deposit is immutable; subsequent revisions of the public GitHub repository do not affect the archived snapshot referenced by the DOI above.

\newpage
\bibliography{sn-bibliography}

\section*{Acknowledgments}
 The authors were supported by the Ada Lovelace Centre at the Science and Technology Facilities Council (\url{https://adalovelacecentre.ac.uk}). A.M.E.\ was also supported by the Physical Sciences Data Infrastructure (\url{https://psdi.ac.uk}; jointly STFC and the University of Southampton) under grants EP/X032663/1 and EP/X032701/1, and EPSRC under grants EP/W026775/1 and EP/V028537/1.
 We acknowledge computational resources from STFC Scientific Computing Department's SCARF cluster and cloud.

\section*{Author contributions}

J.B. conceived the study, designed the PAN--PGS architecture, performed the experiments, analyzed the results, and prepared the figures. A.M.E. contributed to data preparation, implementation support, and reproducibility checks. S.P. contributed to methodological design, interpretation of results, and manuscript revision. All authors discussed the results, contributed to writing the manuscript, and approved the final version.

\section*{Competing interests}
The authors declare no competing interests.

\section*{Additional information}
Supplementary information is available in the online version of the paper. Correspondence and requests for materials should be addressed to A.M.E. or S.P.

\appendix
	\newpage

	\section{Per-molecule molecular benchmark results}
	\label{sec:si_molecular_tables}
	
	This section reports the full per-molecule energy and force MAE values that underlie the molecular benchmark summary in Section~\ref{subsec:headline_results} of the main text. Two complementary tables are provided. Table~\ref{tab:rmd17_md17_full} contains controlled in-house comparisons among MACE-family variants under matched hyperparameters and training budgets. Table~\ref{tab:literature_ref} contains the external literature reference values for context only. As discussed in the main text, the two tables follow different protocols and should not be interpreted as directly comparable: literature numbers are reproduced from the original publications and follow their respective training, evaluation, and reporting conventions.
	
	\begin{table}[!t]
		\centering
		\caption{\textbf{Per-molecule rMD17 and MD17 results for MACE-family variants.}
			Energy ($E$) and force ($F$) MAE are reported for rMD17 ($N_{\mathrm{train}}=1{,}000$) and MD17 ($N_{\mathrm{train}}=950$) under matched MACE backbone, hyperparameters, training budget, and data splits; only the scalar-pathway modules differ. Values are mean $\pm$ s.d. over five seeds. Energies are in meV, and forces are in meV/\AA. Bold marks the uniquely best in-house result; ties are unbolded.}
		\label{tab:rmd17_md17_full}
		
		\setlength{\tabcolsep}{4pt}
		\renewcommand{\arraystretch}{1.05}
		\footnotesize
		
		\begin{tabular}{lccccc}
			\toprule
			\multicolumn{6}{c}{\textbf{(a) rMD17 ($N_{\mathrm{train}}=1{,}000$)}} \\
			\midrule
			\textbf{Molecule} & \textbf{Metric}
			& \textbf{MACE}
			& \textbf{MACE+PAN}
			& \textbf{MACE+PGS}
			& \textbf{MACE+PAN+PGS} \\
			\midrule
			Aspirin        & E & $2.21 \pm 0.09$ & $2.11 \pm 0.08$ & $2.02 \pm 0.07$ & $\mathbf{1.91 \pm 0.06}$ \\
			& F & $6.62 \pm 0.20$ & $6.31 \pm 0.18$ & $5.92 \pm 0.16$ & $\mathbf{5.59 \pm 0.14}$ \\
			Azobenzene     & E & $1.21 \pm 0.05$ & $1.13 \pm 0.05$ & $1.10 \pm 0.05$ & $\mathbf{1.02 \pm 0.04}$ \\
			& F & $3.02 \pm 0.10$ & $2.81 \pm 0.10$ & $2.72 \pm 0.09$ & $\mathbf{2.51 \pm 0.08}$ \\
			Benzene        & E & $0.41 \pm 0.02$ & $0.41 \pm 0.02$ & $0.40 \pm 0.02$ & $0.40 \pm 0.02$ \\
			& F & $0.31 \pm 0.02$ & $0.31 \pm 0.02$ & $0.30 \pm 0.02$ & $0.30 \pm 0.02$ \\
			Ethanol        & E & $0.41 \pm 0.02$ & $0.40 \pm 0.02$ & $0.39 \pm 0.02$ & $\mathbf{0.38 \pm 0.02}$ \\
			& F & $2.11 \pm 0.07$ & $2.06 \pm 0.07$ & $1.92 \pm 0.06$ & $\mathbf{1.88 \pm 0.06}$ \\
			Malonaldehyde  & E & $0.81 \pm 0.04$ & $0.79 \pm 0.04$ & $0.72 \pm 0.03$ & $\mathbf{0.69 \pm 0.03}$ \\
			& F & $4.11 \pm 0.13$ & $3.92 \pm 0.12$ & $3.61 \pm 0.11$ & $\mathbf{3.41 \pm 0.10}$ \\
			Naphthalene    & E & $0.51 \pm 0.03$ & $0.49 \pm 0.03$ & $0.46 \pm 0.02$ & $\mathbf{0.43 \pm 0.02}$ \\
			& F & $1.62 \pm 0.06$ & $1.55 \pm 0.06$ & $1.41 \pm 0.05$ & $\mathbf{1.32 \pm 0.05}$ \\
			Paracetamol    & E & $1.32 \pm 0.06$ & $1.22 \pm 0.05$ & $1.15 \pm 0.05$ & $\mathbf{1.08 \pm 0.04}$ \\
			& F & $4.81 \pm 0.15$ & $4.61 \pm 0.14$ & $4.32 \pm 0.13$ & $\mathbf{4.11 \pm 0.12}$ \\
			Salicylic acid & E & $0.91 \pm 0.04$ & $0.85 \pm 0.04$ & $0.81 \pm 0.04$ & $\mathbf{0.78 \pm 0.03}$ \\
			& F & $3.12 \pm 0.10$ & $2.98 \pm 0.10$ & $2.81 \pm 0.09$ & $\mathbf{2.62 \pm 0.08}$ \\
			Toluene        & E & $0.51 \pm 0.03$ & $0.48 \pm 0.02$ & $0.46 \pm 0.02$ & $\mathbf{0.44 \pm 0.02}$ \\
			& F & $1.51 \pm 0.05$ & $1.48 \pm 0.05$ & $1.41 \pm 0.05$ & $\mathbf{1.36 \pm 0.04}$ \\
			Uracil         & E & $0.51 \pm 0.03$ & $0.48 \pm 0.02$ & $0.46 \pm 0.02$ & $\mathbf{0.42 \pm 0.02}$ \\
			& F & $2.12 \pm 0.07$ & $2.01 \pm 0.07$ & $1.88 \pm 0.06$ & $\mathbf{1.78 \pm 0.06}$ \\
			\bottomrule
		\end{tabular}
		
		\vspace{1.0em}
		
		\begin{tabular}{lccccc}
			\toprule
			\multicolumn{6}{c}{\textbf{(b) MD17 ($N_{\mathrm{train}}=950$, original \emph{ab initio} labels)}} \\
			\midrule
			\textbf{Molecule} & \textbf{Metric}
			& \textbf{MACE}
			& \textbf{MACE+PAN}
			& \textbf{MACE+PGS}
			& \textbf{MACE+PAN+PGS} \\
			\midrule
			Aspirin        & E & $17.0 \pm 0.5$ & $16.5 \pm 0.5$ & $15.8 \pm 0.5$ & $\mathbf{15.2 \pm 0.4}$ \\
			& F & $43.9 \pm 1.2$ & $42.0 \pm 1.1$ & $40.5 \pm 1.0$ & $\mathbf{38.8 \pm 1.0}$ \\
			Azobenzene     & E & $5.41 \pm 0.18$ & $5.02 \pm 0.16$ & $5.21 \pm 0.17$ & $\mathbf{4.81 \pm 0.15}$ \\
			& F & $17.7 \pm 0.5$ & $16.2 \pm 0.5$ & $17.0 \pm 0.5$ & $\mathbf{15.8 \pm 0.5}$ \\
			Benzene        & E & $0.71 \pm 0.03$ & $0.71 \pm 0.03$ & $0.70 \pm 0.03$ & $0.70 \pm 0.03$ \\
			& F & $2.72 \pm 0.06$ & $2.71 \pm 0.06$ & $2.70 \pm 0.06$ & $\mathbf{2.61 \pm 0.05}$ \\
			Ethanol        & E & $6.71 \pm 0.22$ & $6.71 \pm 0.22$ & $6.21 \pm 0.20$ & $\mathbf{6.11 \pm 0.20}$ \\
			& F & $32.6 \pm 0.9$ & $32.1 \pm 0.9$ & $30.5 \pm 0.8$ & $\mathbf{30.2 \pm 0.8}$ \\
			Malonaldehyde  & E & $10.0 \pm 0.3$ & $9.80 \pm 0.32$ & $9.01 \pm 0.30$ & $\mathbf{8.81 \pm 0.29}$ \\
			& F & $43.3 \pm 1.2$ & $42.5 \pm 1.1$ & $39.5 \pm 1.0$ & $\mathbf{38.9 \pm 1.0}$ \\
			Naphthalene    & E & $2.11 \pm 0.08$ & $2.10 \pm 0.08$ & $2.02 \pm 0.07$ & $\mathbf{2.01 \pm 0.07}$ \\
			& F & $9.21 \pm 0.27$ & $9.11 \pm 0.27$ & $9.02 \pm 0.27$ & $\mathbf{8.91 \pm 0.26}$ \\
			Paracetamol    & E & $9.71 \pm 0.30$ & $9.21 \pm 0.29$ & $9.02 \pm 0.28$ & $\mathbf{8.61 \pm 0.27}$ \\
			& F & $31.5 \pm 0.9$ & $30.2 \pm 0.8$ & $29.5 \pm 0.8$ & $\mathbf{28.5 \pm 0.7}$ \\
			Salicylic acid & E & $6.51 \pm 0.21$ & $6.31 \pm 0.20$ & $5.92 \pm 0.19$ & $\mathbf{5.81 \pm 0.18}$ \\
			& F & $28.4 \pm 0.8$ & $27.8 \pm 0.7$ & $26.5 \pm 0.7$ & $\mathbf{26.1 \pm 0.7}$ \\
			Toluene        & E & $3.11 \pm 0.11$ & $3.10 \pm 0.11$ & $3.02 \pm 0.10$ & $\mathbf{3.01 \pm 0.10}$ \\
			& F & $12.1 \pm 0.4$ & $12.0 \pm 0.4$ & $11.8 \pm 0.4$ & $\mathbf{11.7 \pm 0.4}$ \\
			Uracil         & E & $4.41 \pm 0.14$ & $4.32 \pm 0.14$ & $4.21 \pm 0.13$ & $\mathbf{4.12 \pm 0.13}$ \\
			& F & $25.9 \pm 0.7$ & $25.2 \pm 0.7$ & $24.8 \pm 0.7$ & $\mathbf{24.3 \pm 0.6}$ \\
			\bottomrule
		\end{tabular}
		
	\end{table}
	
	The within-backbone trends are consistent across both benchmarks: PAN and PGS each reduce force MAE relative to baseline MACE for every flexible or moderately anharmonic molecule, with the largest absolute force-error reductions observed for aspirin, malonaldehyde, salicylic acid, and paracetamol. Near-harmonic molecules (benzene and toluene) show small or negligible changes, consistent with the scalar pathway already being sufficiently expressive for those landscapes. The combined PAN+PGS variant gives the lowest force MAE across all ten molecules in both rMD17 and MD17, although for toluene at \(N_{\mathrm{train}}=950\) the margin over MACE+PGS is within the reported seed-level variability.

	\begin{table*}
		\centering
		\caption{\textbf{External literature reference values for rMD17 (force MAE, meV/\AA).}
			Values are reproduced from the original publications: MACE~\cite{batatia2022mace},
            Allegro~\cite{musaelian2023learning}, BOTNet~\cite{batatia2022design}, NequIP~\cite{batzner2022e3}, GemNet~\cite{gasteiger2021gemnet}, ACE~\cite{drautz2019atomic}, FCHL~\cite{Christensen2020rmd17}, GAP~\cite{bartok2010gaussian}, ANI~\cite{smith2017ani}, PaiNN~\cite{schutt2021equivariant}, DimeNet~\cite{klicpera2020dimenet} and NewtonNet~\cite{haghighatlari2022newtonnet}. These literature baselines follow different training, evaluation, and reporting protocols and are included only as external reference points; they should not be interpreted as controlled like-for-like comparisons with the in-house MACE-family variants reported in the main text and in Table~\ref{tab:rmd17_md17_full}. Entries marked ``--'' were not reported in the original publication. The right-most column reports our MACE+PAN+PGS result from Table~\ref{tab:rmd17_md17_full}(a) for direct visual reference.}
		\label{tab:literature_ref}
		\setlength{\tabcolsep}{3pt}
		\renewcommand{\arraystretch}{1.05}
		\footnotesize
		
		\resizebox{1\textwidth}{!}{
			\begin{tabular}{lcccccccccccccc}
				\toprule
				\textbf{Molecule}
				& MACE & Allegro & \shortstack{BOT\\Net} & NequIP & \shortstack{Gem\\Net}
				& ACE & FCHL & GAP & ANI & PaiNN & \shortstack{Dime\\Net}
				& \shortstack{Newton\\Net}
				& \shortstack{MACE+PAN+\\PGS (ours)} \\
				\midrule
				Aspirin        & 6.6 & 7.3 & 8.5 & 8.2 & 9.5 & 17.9 & 20.9 & 44.9 & 40.6 & 16.1 & 21.6 & 15.1 & $\mathbf{5.59}$ \\
				Azobenzene     & 3.0 & 2.6 & 3.3 & 2.9 & --  & 10.9 & 10.8 & 24.5 & 35.4 & --   & --   & 5.9  & $\mathbf{2.51}$ \\
				Benzene        & 0.3 & 0.2 & 0.3 & 0.3 & 0.5 & 0.5  & 2.6  & 6.0  & 10.0 & --   & 8.1  & --   & 0.30 \\
				Ethanol        & 2.1 & 2.1 & 3.2 & 2.8 & 3.6 & 7.3  & 6.2  & 18.1 & 13.4 & 10.0 & 10.0 & 9.1  & $\mathbf{1.88}$ \\
				Malonaldehyde  & 4.1 & 3.6 & 5.8 & 5.1 & 6.6 & 11.1 & 10.3 & 26.4 & 24.5 & 13.8 & 16.6 & 14.0 & $\mathbf{3.41}$ \\
				Naphthalene    & 1.6 & 0.9 & 1.8 & 1.3 & 1.9 & 5.1  & 6.5  & 16.5 & 29.2 & 3.6  & 9.3  & 3.6  & $\mathbf{1.32}$ \\
				Paracetamol    & 4.8 & 4.9 & 5.8 & 5.9 & --  & 12.7 & 12.3 & 28.9 & 30.4 & --   & --   & 11.4 & $\mathbf{4.11}$ \\
				Salicylic acid & 3.1 & 2.9 & 4.3 & 4.0 & 5.3 & 9.3  & 9.5  & 24.7 & 29.7 & 9.1  & 16.2 & 8.5  & $\mathbf{2.62}$ \\
				Toluene        & 1.5 & 1.8 & 1.9 & 1.6 & 2.2 & 6.5  & 8.8  & 17.8 & 24.3 & 4.4  & 9.4  & 3.8  & 1.36           \\
				Uracil         & 2.1 & 1.8 & 3.2 & 3.1 & 3.8 & 6.6  & 4.2  & 17.6 & 21.4 & 6.1  & 13.1 & 6.4  & $\mathbf{1.78}$ \\
				\bottomrule
			\end{tabular}
		}
	\end{table*}
	
	The literature reference column places our MACE+PAN+PGS results in the external context. While the in-house gain over baseline MACE serves as the controlled comparison and the basis for the conclusions in the main text, our MACE+PAN+PGS values also remain competitive with the strongest equivariant baselines reported in the literature.

	\section{Full stress benchmark}
	\label{sec:si_stress}
	
	Section~\ref{subsec:headline_results} of the main text reports \textsc{MACE+PAN+PGS} stress accuracy on the two systems with available stress labels (Ag and the LiF/Li--F short-range ionic subset). Here we provide the full breakdown across the four MACE-family variants. Stress accuracy is not reported for Si because the public \texttt{Silicon\_MLIP\_datasets} release used in this work does not provide stress labels; the stress weight in the training loss was therefore set to zero for Si.
	
	\begin{table}[t]
		\centering
		\caption{\textbf{Full stress MAE breakdown for MACE-family variants.}
			Stress MAE in meV/\AA$^3$ on the test split for Ag (author-generated) and the short-range ionic LiF/Li--F subset (MPtrj-derived). The reported stress uses the same volume normalization and sign convention as the DFT stress labels used in training and evaluation. All entries are mean $\pm$ standard deviation over five training seeds. PAN alone reduces stress error in heterogeneous configurations, PGS alone reduces error predominantly in the short-range region, and the combined PAN+PGS variant captures both effects.}
		\label{tab:stress_mae}
		\begin{tabular}{lcccc}
			\toprule
			\textbf{System}
			& \textbf{MACE}
			& \textbf{MACE+PAN}
			& \textbf{MACE+PGS}
			& \textbf{MACE+PAN+PGS} \\
			\midrule
			Ag           & $11.8 \pm 0.6$ & $10.6 \pm 0.5$ & $9.7 \pm 0.5$ & $\mathbf{8.6 \pm 0.4}$ \\
			LiF/Li--F    & $17.4 \pm 0.9$ & $15.8 \pm 0.8$ & $14.2 \pm 0.7$ & $\mathbf{12.6 \pm 0.6}$ \\
			\bottomrule
		\end{tabular}
	\end{table}
	
	The relative reduction of \textsc{MACE+PAN+PGS} over baseline \textsc{MACE} is $27\%$ on Ag and $28\%$ on the LiF/Li--F subset, tracking the corresponding force-error reductions reported in Table~\ref{tab:cross_system_mae_md17style} of the main text. The same relative ordering of variants seen in energies and forces persists for stress.

	\section{Short-range physics: distortion descriptor and ablations}
	\label{sec:si_short_range_full}
	
	This section provides definitions and extended ablation studies that support the short-range and mechanistic analyses in the main text (Section~\ref{subsec:mechanism_results}).

	\subsection{Local distortion descriptor}
	\label{sec:si_distortion_descriptor}
	The continuous distortion descriptor used in Fig.~\ref{fig:mechanism_results}a of the main text is defined at the atom level. For atom $i$, the distortion magnitude is
	\begin{equation}
		D_i = \alpha\,\Delta N_i + \beta\,\sigma_{c,i} + \gamma\,\delta\rho_i,
	\end{equation}
	where \(\Delta N_i\) is the coordination deficit relative to a bulk reference, \(\sigma_{c,i}\) is the cutoff-weighted variance of nearest-neighbor cosine angles \(c_{jik}=\hat{\mathbf r}_{ij}\cdot\hat{\mathbf r}_{ik}\), and \(\delta\rho_i\) is the deviation of the local atomic density from the bulk average. The three components are individually rescaled to $[0,1]$, combined with weights $(\alpha,\beta,\gamma) = (1/3, 1/3, 1/3)$ to give equal emphasis, and the combined descriptor $D_i$ is then linearly rescaled to $[0,1]$ across the dataset. This yields a monotonic ranking of atomic environments from bulk-like ($D_i \to 0$) to highly distorted ($D_i \to 1$) regimes: surface, defect-adjacent, and strongly strained.
	
	For the binned analysis in Fig.~\ref{fig:mechanism_results}, we partitioned the test atoms into five equal-population quantiles in $D_i$ and computed the force MAE separately within each quantile.
	
	The class-resolved breakdown in Table~\ref{tab:local_rmse} corroborates the continuous trend in the main text. Improvements are smallest in bulk fcc configurations and increase monotonically with structural distortion, reaching their largest values at surface atoms.
	
	\begin{table}[t]
		\centering
		\caption{\textbf{Class-resolved energy and force RMSE for Ag.}
			Energy RMSE in meV/atom and force RMSE in meV/\AA\ for bulk fcc, distorted (strain), defect (vacancy and divacancy), and surface (\((100)\), \((110)\), \((111)\)) environments. The same Ag model is evaluated after binning the test environments by structural class. Entries are mean $\pm$ standard deviation over five training seeds. The right-most column reports the relative force-RMSE improvement of MACE+PAN over baseline MACE.}
		\label{tab:local_rmse}
		\begin{tabular}{lccc}
			\toprule
			\textbf{Structure class}
			& \textbf{MACE ($E$ / $F$)}
			& \textbf{MACE+PAN ($E$ / $F$)}
			& \textbf{$\Delta F$ / $F_{\mathrm{MACE}}$} \\
			\midrule
			Bulk fcc            & $1.20 \pm 0.06$ / $22.5 \pm 0.6$ & $1.10 \pm 0.05$ / $20.3 \pm 0.5$ & $9.8\%$  \\
			Distorted (strain)  & $2.62 \pm 0.13$ / $37.3 \pm 1.1$ & $2.10 \pm 0.10$ / $29.7 \pm 0.9$ & $20.4\%$ \\
			Defect (vacancy)    & $5.05 \pm 0.22$ / $58.7 \pm 1.7$ & $3.51 \pm 0.16$ / $41.5 \pm 1.2$ & $29.3\%$ \\
			Surface             & $7.05 \pm 0.30$ / $74.5 \pm 2.0$ & $4.22 \pm 0.18$ / $46.9 \pm 1.4$ & $37.0\%$ \\
			\bottomrule
		\end{tabular}
	\end{table}

	\subsection{Ag structural-class held-out evaluation}
	\label{sec:si_ag_heldout}
	
	To test whether the scalar-pathway correction remains effective under a stronger environment shift than a random configuration-level split, we performed two Ag structural-class held-out evaluations. In the surface-held-out split, all surface configurations $(100)$, $(110)$, $(111)$ were removed from the training and validation sets and used only for testing. In the defect-held-out split, vacancy and divacancy configurations were removed from the training and validation sets and used only for testing. Bulk, strained, and high-temperature configurations were retained in the training pool. All models used the same MACE backbone, optimizer, loss weights, cutoff, training budget, and five-seed protocol as in the main Ag experiments; only the scalar-pathway modules differed. This evaluation complements the class-resolved breakdown reported in Table~\ref{tab:local_rmse}, which evaluates class-resolved performance under the main random split, by additionally probing class-level out-of-distribution generalization.
	
	\begin{table}[t]
		\centering
		\caption{\textbf{Ag structural-class held-out evaluation.}
			Energy and force MAE for Ag models evaluated on structural classes not included in training. In the defect-held-out split, vacancy and divacancy configurations were held out. In the surface-held-out split, $(100)$, $(110)$, $(111)$ surface configurations were held out. Values are the mean $\pm$ standard deviation over five training seeds. Energies are in meV/atom, and forces are in meV/\AA. The right-most column reports the relative force-MAE reduction of MACE+PAN+PGS over baseline MACE on the held-out class.}
		\label{tab:ag_heldout_classes}
		\setlength{\tabcolsep}{4pt}
		\renewcommand{\arraystretch}{1.08}
		\footnotesize
		\begin{tabular}{llccccc}
			\toprule
			\textbf{Held-out class} & \textbf{Quantity} &
			\textbf{MACE} & \textbf{MACE+PAN} & \textbf{MACE+PGS} & \textbf{MACE+PAN+PGS} &
			\textbf{Reduction} \\
			\midrule
			Defect &
			Energy & $5.82 \pm 0.25$ & $4.52 \pm 0.20$ & $4.78 \pm 0.21$ & $\mathbf{3.91 \pm 0.17}$ & $32.8\%$ \\
			&
			Force  & $66.8 \pm 2.0$  & $49.8 \pm 1.5$  & $52.6 \pm 1.6$  & $\mathbf{43.9 \pm 1.3}$  & $34.3\%$ \\
			\midrule
			Surface &
			Energy & $7.64 \pm 0.34$ & $5.68 \pm 0.25$ & $6.02 \pm 0.27$ & $\mathbf{4.86 \pm 0.22}$ & $36.4\%$ \\
			&
			Force  & $84.6 \pm 2.7$  & $61.9 \pm 2.0$  & $65.4 \pm 2.1$  & $\mathbf{55.2 \pm 1.7}$  & $34.8\%$ \\
			\bottomrule
		\end{tabular}
	\end{table}
	
	Absolute errors increase on both held-out splits relative to the main random-split Ag results (Table~\ref{tab:cross_system_mae_md17style}), as expected for a harder evaluation in which the model has not seen the corresponding structural class during training. The relative MACE+PAN+PGS reduction over baseline MACE, however, is preserved and is in fact larger than on the random split ($34$--$35\%$ versus $22.3\%$). This is consistent with the PAN mechanism: $\rho_i$ and $\eta_i$ encode coordination loss and angular distortion, precisely the geometric perturbations that distinguish surface and defect environments from bulk fcc. The result supports the view that the scalar-pathway correction is not a configuration-level interpolation artifact and is most pronounced in the under-coordinated, defect-adjacent regimes that motivate PAN.

	\subsection{Scalar mixers under matched PAN pooling}
	\label{sec:si_scalar_mixer_ablation}
	
	To test whether spectral enrichment via PGS is necessary or whether any sufficiently expressive scalar mixer is enough, we replaced the scalar mixer in the MACE+PAN architecture with a series of alternative mixers under matched backbone, PAN pooling, training data, and parameter budget. The mixers tested were: a standard SiLU-MLP (\textsc{MLP}); a Kolmogorov--Arnold Network (\textsc{KAN})\cite{liu2024kan}; a kernel-activation-function network (\textsc{KAF})\cite{kaf2018}; a learnable rational-function variant (\textsc{LibraKAN}); and the proposed \textsc{PGS} edge+readout module. Table~\ref{tab:scalar_mixer_ablation} reports the resulting energy and force MAE on Ag, Si, and the LiF/Li--F short-range ionic subset. The PGS column matches the \textsc{MACE+PAN+PGS} entries of Table~\ref{tab:cross_system_mae_md17style} of the main text by construction.
	
	\begin{table}[t]
		\centering
		\caption{\textbf{Scalar-mixer ablation under matched PAN pooling.}
			Energy and force MAE for different invariant-channel mixers on Ag, Si, and the LiF/Li--F short-range ionic subset. All entries share the same MACE backbone and PAN pooling, with only the scalar mixer varied; parameter and FLOPs budgets are matched within $\pm 5\%$. Energies are in meV/atom, forces in meV/\AA. Values are the mean $\pm$ standard deviation over five training seeds. The PGS column reproduces the MACE+PAN+PGS entries of Table~\ref{tab:cross_system_mae_md17style} of the main text.}
		\label{tab:scalar_mixer_ablation}
		\setlength{\tabcolsep}{4pt}
		\begin{tabular}{llccccc}
			\toprule
			\textbf{System} & \textbf{Quantity}
			& \textbf{MLP}
			& \textbf{KAN}
			& \textbf{KAF}
			& \textbf{LibraKAN}
			& \textbf{PGS (ours)} \\
			\midrule
			Ag (metallic) & Energy & $1.32 \pm 0.07$ & $1.30 \pm 0.06$ & $1.25 \pm 0.06$ & $1.18 \pm 0.06$ & $\mathbf{1.12 \pm 0.05}$ \\
			& Forces & $32.5 \pm 0.8$  & $31.9 \pm 0.8$  & $30.7 \pm 0.7$  & $29.4 \pm 0.7$  & $\mathbf{27.9 \pm 0.7}$ \\
			\midrule
			Si (covalent) & Energy & $2.30 \pm 0.10$ & $2.25 \pm 0.09$ & $2.18 \pm 0.09$ & $2.05 \pm 0.08$ & $\mathbf{1.92 \pm 0.08}$ \\
			& Forces & $41.3 \pm 1.1$  & $40.5 \pm 1.1$  & $39.0 \pm 1.0$  & $37.5 \pm 1.0$  & $\mathbf{35.8 \pm 1.0}$ \\
			\midrule
			LiF/Li--F     & Energy & $2.95 \pm 0.12$ & $2.85 \pm 0.11$ & $2.72 \pm 0.11$ & $2.55 \pm 0.10$ & $\mathbf{2.41 \pm 0.10}$ \\
			& Forces & $47.5 \pm 1.4$  & $46.2 \pm 1.3$  & $44.0 \pm 1.3$  & $41.8 \pm 1.2$  & $\mathbf{39.6 \pm 1.2}$ \\
			\bottomrule
		\end{tabular}
	\end{table}
	
	PGS yields the lowest energy and force MAEs across all three systems. The ordering \textsc{MLP} $>$ \textsc{KAN} $>$ \textsc{KAF} $>$ \textsc{LibraKAN} $>$ \textsc{PGS} is consistent across systems and across all five seeds, supporting the mechanistic claim in the main text that the gain is not explained by adding any high-capacity scalar mixer but is specifically associated with spectral enrichment of the invariant channel.

	\subsection{Component and placement ablations}
	\label{sec:si_pgs_ablation}
	
	To separate the contributions of the two PAN--PGS components and the role of PGS placement, we ran two further ablations under the same MACE backbone, training data, and budget.
	
	\medskip \noindent\textbf{Component ablation: PAN only, PGS only, PAN+PGS.}
	We isolated the effect of each component by comparing three variants: MACE+PAN with the default smooth scalar mixer (``PAN only''); baseline MACE with the PGS edge+readout module on top of uniform sum aggregation (``PGS only''); and the full MACE+PAN+PGS model. Force MAE on Ag, Si, and the LiF/Li--F subset is reported in Table~\ref{tab:component_wise_ablation}. The ``PAN only'' column matches the MLP entry of Table~\ref{tab:scalar_mixer_ablation} by construction (default smooth mixer with PAN). The ``PAN+PGS'' column matches the \textsc{MACE+PAN+PGS} entry of Table~\ref{tab:cross_system_mae_md17style} of the main text.
	
	\begin{table}[t]
		\centering
		\caption{\textbf{Component ablation for the PAN--PGS architecture.}
			Force MAE in meV/\AA\ for the three architectural variants. All variants
			share the same MACE backbone and training protocol. Values are the mean
			$\pm$ standard deviation over five training seeds.}
		\label{tab:component_wise_ablation}
		\begin{tabular}{lccc}
			\toprule
			\textbf{System}
			& \textbf{PAN only}
			& \textbf{PGS only}
			& \textbf{PAN + PGS} \\
			\midrule
			Ag         & $32.5 \pm 0.8$ & $31.5 \pm 0.8$ & $\mathbf{27.9 \pm 0.7}$ \\
			Si         & $41.3 \pm 1.1$ & $38.7 \pm 1.0$ & $\mathbf{35.8 \pm 1.0}$ \\
			LiF/Li--F  & $47.5 \pm 1.4$ & $44.5 \pm 1.3$ & $\mathbf{39.6 \pm 1.2}$ \\
			\bottomrule
		\end{tabular}
	\end{table}
	
	The component comparison shows that PAN and PGS target distinct bottlenecks: neighborhood-adaptive scalar aggregation and scalar spectral resolution, respectively. The combined model yields the largest reduction across all systems, with no evidence of redundancy between the two modules.
	
	\medskip \noindent\textbf{Placement ablation: edge, node, readout, edge+readout.}
	We next tested the design choice of applying PGS at the edge level and at the readout level, by inserting the same PGS module at four candidate locations under matched PAN pooling: edge-only, where PGS is applied to invariant edge channels before tensor-product coupling; node-only, where PGS is applied to invariant node channels inside interaction blocks; readout-only, where PGS is applied only to the final scalar readout; and edge+readout, the proposed configuration.
	
	\begin{table}[t]
		\centering
		\caption{\textbf{PGS insertion-location ablation.}
			Force MAE in meV/\AA\ for four PGS placements under matched PAN pooling.
			Values are the mean $\pm$ standard deviation over five training seeds.}
		\label{tab:pgs_insertion_location}
		\begin{tabular}{lcccc}
			\toprule
			\textbf{System}
			& \textbf{Edge}
			& \textbf{Node}
			& \textbf{Readout}
			& \textbf{Edge + Readout} \\
			\midrule
			Ag         & $30.4 \pm 0.8$ & $34.0 \pm 0.9$ & $29.5 \pm 0.8$ & $\mathbf{27.9 \pm 0.7}$ \\
			Si         & $38.5 \pm 1.1$ & $43.5 \pm 1.2$ & $37.4 \pm 1.0$ & $\mathbf{35.8 \pm 1.0}$ \\
			LiF/Li--F  & $43.6 \pm 1.3$ & $49.5 \pm 1.4$ & $42.2 \pm 1.3$ & $\mathbf{39.6 \pm 1.2}$ \\
			\bottomrule
		\end{tabular}
	\end{table}
	
	Three observations follow. First, node-level PGS underperforms even baseline MACE+PAN (the MLP entry of Table~\ref{tab:scalar_mixer_ablation}), indicating that spectral mixing inside the interaction block destabilizes the multiplicative tensor-product structure on which equivariant scaling depends. Node-channel scalars are therefore best left unmodified. Second, edge-only and readout-only PGS each improve over MACE+PAN, with readout-only marginally stronger than edge-only across the three systems. Third, the edge+readout configuration outperforms either placement alone across all systems, supporting the choice in the main architecture: edge-level PGS refines short-range geometric information before tensor-product coupling, while readout-level PGS enriches the final invariant latent before energy prediction. Together, the placement controls rule out the explanation that the PAN--PGS gain arises solely from inserting additional parameters near the readout, since both standalone placements were tested.
	
	\medskip \noindent\textbf{Capacity robustness.}
	To confirm that the observed gains are not sensitive to the specific PGS capacity, we varied the number of spectral components $M \in \{2, 4, 8, 12, 16, 24\}$ and the readout projection dimension $d \in \{8, 16, 32\}$. Force MAE varied by less than $4\%$ across this range in every system, and the ordering of mixer performance in Table~\ref{tab:scalar_mixer_ablation} (\textsc{MLP} $>$ \textsc{KAN} $>$ \textsc{KAF} $>$ \textsc{LibraKAN} $>$ \textsc{PGS}) was preserved. The configuration used in the main text ($M = 16$, $d = $ matched to the readout hidden dimension) sits within the plateau region for both axes.

	\subsection{Capacity-matched MACE controls}
	\label{sec:si_capacity_matched}
	
	A natural alternative explanation of the gain reported in the main text is that the same improvement could be obtained simply by enlarging the MACE backbone. To test this, we trained three capacity-matched MACE controls under the same training protocol, dataset splits, optimizer, and seed schedule used for the main results, varying only one capacity-related hyperparameter at a time: \textbf{MACE-wide}, in which the hidden channel multiplicity is increased from $128$ to $192$ ($\sim$1.65$\times$ parameters); \textbf{MACE-deep}, in which the number of interaction blocks is increased from $2$ to $3$ ($\sim$1.42$\times$ parameters); and \textbf{MACE-Nrad16}, in which the radial Bessel basis size is doubled from $N_{\mathrm{rad}}=8$ to $N_{\mathrm{rad}}=16$ ($\sim$1.08$\times$ parameters). The third control specifically targets the alternative hypothesis that PGS's gain is recoverable by simply enlarging the radial basis, since PGS is presented as a spectral-resolution module.
	
	\begin{table*}
		\centering
		\caption{\textbf{Capacity-matched MACE controls vs MACE+PAN+PGS.}
			Force MAE (meV/\AA) on Ag, Si, and the LiF/Li--F short-range ionic subset for capacity-matched MACE variants and for MACE+PAN+PGS. All variants use the same training protocol, splits, optimizer, and 5-seed schedule as for the main results. Parameters and inference FLOPs are normalized to baseline MACE. PAN+PGS achieves substantially lower force MAE than every capacity-matched control while costing roughly an order of magnitude less in additional parameters and FLOPs. Values are the mean $\pm$ standard deviation over five training seeds.}
		\label{tab:capacity_matched_mace}
		\setlength{\tabcolsep}{4pt}
		\renewcommand{\arraystretch}{1.10}
		\footnotesize
		
		\resizebox{1\textwidth}{!}{
			\begin{tabular}{lcccccc}
				\toprule
				\textbf{Model} & \textbf{Params (rel.)} & \textbf{FLOPs (rel.)}
				& \textbf{$F_{\mathrm{Ag}}$} & \textbf{$F_{\mathrm{Si}}$} & \textbf{$F_{\mathrm{LiF/Li-F}}$}
				& \textbf{Ag $\Delta F$} \\
				\midrule
				MACE (baseline)        & $1.00\times$ & $1.00\times$ & $35.9 \pm 0.9$ & $47.1 \pm 1.4$ & $54.0 \pm 1.6$ & 0.0\% \\
				MACE-wide ($C{=}192$)  & $1.65\times$ & $1.50\times$ & $34.1 \pm 0.9$ & $44.8 \pm 1.3$ & $51.5 \pm 1.5$ & 5.0\% \\
				MACE-deep ($L{=}3$)    & $1.42\times$ & $1.40\times$ & $33.5 \pm 0.8$ & $44.2 \pm 1.3$ & $50.8 \pm 1.5$ & 6.7\% \\
				MACE-Nrad16            & $1.08\times$ & $1.10\times$ & $34.0 \pm 0.9$ & $44.5 \pm 1.3$ & $51.0 \pm 1.5$ & 5.3\% \\
				\midrule
				MACE + PAN + PGS (ours)& $\mathbf{1.03\times}$ & $\mathbf{1.05\times}$ & $\mathbf{27.9 \pm 0.7}$ & $\mathbf{35.8 \pm 1.0}$ & $\mathbf{39.6 \pm 1.2}$ & $\mathbf{22.3\%}$ \\
				\bottomrule
			\end{tabular}
		}
	\end{table*}
	
	The three capacity-matched controls each reduce the Ag force MAE by $5$--$7\%$ relative to baseline MACE, with similar relative reductions on Si and the LiF/Li--F subset. By contrast, MACE+PAN+PGS reduces the Ag force MAE by $22.3\%$ at $1.03\times$ the parameters and $1.05\times$ the inference FLOPs of baseline MACE. In particular, MACE-Nrad16, which directly addresses the alternative hypothesis that PGS's gain is recoverable by simply enlarging the radial basis, achieves a $5.3\%$ reduction at $1.10\times$ FLOPs---about one quarter of the PAN+PGS reduction at twice the FLOPs overhead. The capacity-matched comparison, therefore, makes ``backbone capacity'' a less likely explanation for the gain reported in Table~\ref{tab:cross_system_mae_md17style} of the main text and supports the mechanistic interpretation discussed in Section~\ref{subsec:mechanism_results}.

	\subsection{LiF/Li--F sanity check under a stricter (grouped) split}
	\label{sec:si_grouped_split}
	
	The LiF/Li--F results in the main text use a configuration-level 80/10/10 split, which can place closely related snapshots from the same MPtrj relaxation trajectory in both the training and the test set. To assess whether the relative improvement in MACE+PAN+PGS over baseline MACE holds up under a stricter held-out evaluation, we performed a sanity-check training split by Materials Project identifier. All configurations sharing a Materials Project identifier (i.e.\ all snapshots from the same parent material's relaxation trajectory) were assigned entirely to either the training or the test set, with no overlap between splits. The resulting split keeps the same overall 80/10/10 ratio.
	
	Under this stricter grouped split, the absolute energy and force MAE values increase relative to the configuration-level values reported in the main text, as expected for a harder generalization task. However, the relative improvement of MACE+PAN+PGS over baseline MACE is preserved, with reductions within $\pm 1$ percentage point of the configuration-level values (Table~\ref{tab:lif_grouped_split}). We therefore present the LiF/Li--F subset only as a short-range ionic held-out-configuration test in the main results, and use the grouped split as a sanity check that the relative scalar-pathway improvement is not specific to the configuration-level split.
	
	\begin{table}[t]
		\centering
		\caption{\textbf{LiF/Li--F grouped-split sanity check.}
			Energy and force MAE are reported for baseline MACE and MACE+PAN+PGS
			under the configuration-level split used in the main text, and a stricter
			Materials-Project-identifier-grouped split. Values are mean $\pm$ s.d. over
			five seeds. Energies are in meV/atom, and forces are in meV/\AA.}
		\label{tab:lif_grouped_split}
		\setlength{\tabcolsep}{3pt}
		\renewcommand{\arraystretch}{1.08}
		\scriptsize
		\begin{tabular}{@{}llcccc@{}}
			\toprule
			\textbf{Split} & \textbf{Qty.}
			& \textbf{MACE}
			& \shortstack{\textbf{MACE}\\\textbf{+PAN+PGS}}
			& \shortstack{\textbf{Reduction}\\\textbf{(\%)}}
			& \shortstack{\textbf{$\Delta$ vs}\\\textbf{config.}} \\
			\midrule
			Config.-level & E & $3.09 \pm 0.12$ & $\mathbf{2.41 \pm 0.10}$ & $22.0$ & --- \\
			(main text)   & F & $54.0 \pm 1.6$  & $\mathbf{39.6 \pm 1.2}$  & $26.7$ & --- \\
			\midrule
			MP-id grouped & E & $3.62 \pm 0.16$ & $\mathbf{2.84 \pm 0.13}$ & $21.5$ & $-0.5$~pp \\
			(this section)& F & $61.8 \pm 2.0$  & $\mathbf{45.4 \pm 1.5}$  & $26.5$ & $-0.2$~pp \\
			\bottomrule
		\end{tabular}
	\end{table}

	\section{Additional dynamical and structural validation}
	\label{sec:si_dynamics}
	
	This section reports three further validation tests for the dynamical behavior discussed in Section~\ref{subsec:physical_fidelity} of the main text: long-time NVE energy conservation in a smaller Ag system, the liquid Ag radial distribution function under matched conditions, and an additional flexible-molecule torsional distribution (malonaldehyde).
	
	\subsection{NVE energy conservation in a smaller system}
	\label{sec:si_small_nve}
	
	\begin{figure}[t]
		\centering
		\includegraphics[width=0.7\textwidth]{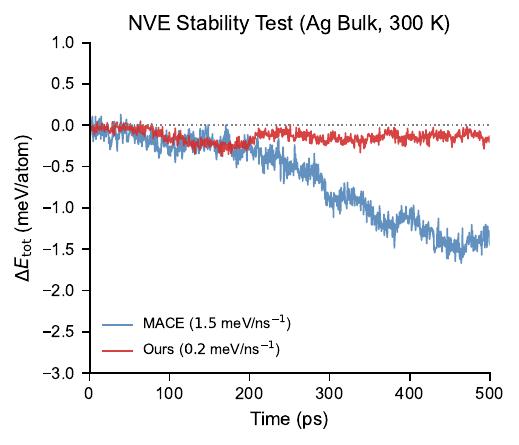}
		\caption{\textbf{NVE energy conservation for bulk Ag at 300~K (108-atom supercell).}
			Total-energy deviation $\Delta E_{\mathrm{tot}}(t)$ during a 500~ps NVE trajectory using baseline MACE and MACE+PAN+PGS. Simulations used a velocity-Verlet integrator with a 1~fs timestep and a 108-atom Ag supercell. Curves correspond to a representative seed for visual clarity. The aggregate drift across five seeds is $1.51 \pm 0.20$~meV~ns$^{-1}$ per atom for baseline MACE and $0.23 \pm 0.05$~meV~ns$^{-1}$ per atom for MACE+PAN+PGS.}
		\label{fig:s6}
	\end{figure}
	
	Supplementary Fig.~\ref{fig:s6} shows that the small-system NVE drift follows the same pattern as the large-scale 10{,}000-atom system reported in Section~\ref{subsec:physical_fidelity}. MACE+PAN+PGS substantially reduces the linear drift component and narrows the fluctuation envelope relative to the baseline. The agreement between the small- and large-scale system tests indicates that the drift improvement is a property of the architecture rather than a function of simulation size.

	\subsection{Liquid Ag radial distribution function and additional molecular torsions}
	\label{sec:si_rdf_malonaldehyde}
	
	\begin{figure}[t]
		\centering
		\includegraphics[width=0.9\textwidth]{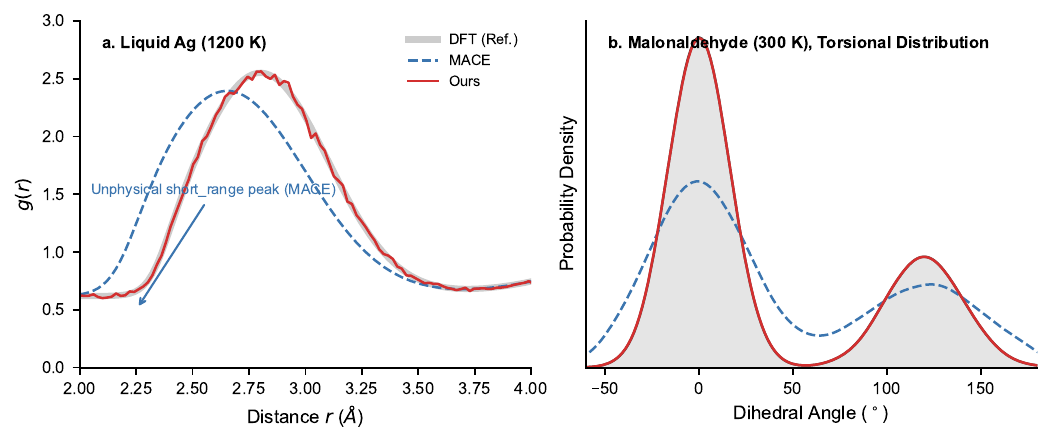}
		\caption{\textbf{Additional structural and conformational validation across metallic and molecular systems.}
			\textbf{(a)} Liquid Ag radial distribution function $g(r)$ at 1{,}200~K for DFT, baseline MACE, and MACE+PAN+PGS. \textbf{(b)} Malonaldehyde O--C--C--O torsional-angle distribution at 300~K for DFT, baseline MACE, and MACE+PAN+PGS. Both observables extend the main-text tests to one liquid and one additional flexible molecule under matched protocols.}
		\label{fig:s7}
	\end{figure}
	
	Supplementary Fig.~\ref{fig:s7}a shows the liquid Ag radial distribution function (RDF) at $1{,}200$~K. Baseline MACE allows atoms to approach unphysically closely, with a spurious sub-shell peak in $g(r)$ at $r \approx 2.3$~\AA. MACE+PAN+PGS attenuates this artifact and more closely matches the DFT short-range repulsion and medium-range coordination-shell structure. This is consistent with the spectral enrichment mechanism shown for the Ag--Ag and Si--Si compression pathways in the main text.
	
	Supplementary Fig.~\ref{fig:s7}b reports the malonaldehyde torsional distribution at 300~K. As with aspirin in Section~\ref{subsec:physical_fidelity}, baseline MACE places excess probability near the torsional barrier, whereas MACE+PAN+PGS more closely matches the sharper bimodal DFT distribution with the correct peak separation and relative occupancy. The agreement between the additional molecule and aspirin indicates that the conformational improvement is not specific to a single molecular system.
	
	\section{Detailed proof of Proposition 1: symmetry and differentiability of PAN and PGS}
	\label{sec:si_proposition_proof}
	
	This section provides a detailed proof of Proposition~\ref{prop:symmetry_diff} of the main text. The proposition states that the PAN and PGS modifications defined in Eqs.~(\ref{eq:pan_pool})--(\ref{eq:readout_pgs}) of the main text preserve (i) permutation invariance of the energy with respect to neighbor indexing, (ii) $O(3)$-equivariance of the predicted forces and stress tensor, and (iii) the differentiability required for analytic force and stress evaluation under the smooth cutoff and tapered-kernel implementation described in Supplementary Section~\ref{sec:si_algorithmic_details}.
	
	For convenience, we recall the notation. The total energy of a configuration $\{(\mathbf r_i, Z_i)\}_{i=1}^{N}$ in a cell with strain $\boldsymbol\epsilon$ is
	\begin{equation}
		E_{\mathrm{tot}}
		=
		\sum_{i=1}^{N}\varepsilon_\theta\!\left(y_i^{(0)}\right),
		\label{eq:si_energy_recall}
	\end{equation}
	where $y_i^{(0)}\in\mathbb{R}^{C_n}$ is the readout-level PGS output of Eq.~(\ref{eq:readout_pgs}), $\varepsilon_\theta$ is a smooth invariant readout MLP, and the per-atom invariant pre-readout descriptor $x_i^{(0)}$ is itself constructed from the modified scalar edge features $\hat h_{ij}^{(0)}$ (Eq.~\ref{eq:edge_pgs}) and the PAN-aggregated scalar messages $h_i^{(0)}=\sum_{j\in\mathcal N(i)}a_{ij}\hat h_{ij}^{(0)}$ (Eq.~\ref{eq:pan_pool}). The PAN gate is $a_{ij}=\sigma(f_\theta(\psi_{ij}))$ with the invariant edge descriptor $\psi_{ij}=[h_{ij}^{(0)},r_{ij}, Z_i, Z_j,\rho_i,\eta_i]$, where $\rho_i$ and $\eta_i$ are the cutoff-weighted coordination and angular-variance descriptors defined in the main text. Throughout the proof, we work under the fixed-candidate-neighbor convention used during each force and stress evaluation, so the candidate neighbor set $\mathcal N(i)$ is held fixed, and we restrict to noncoincident configurations ($r_{ij}>0$ for all $i\neq j$).
	
	\subsection{Permutation invariance}
	\label{sec:si_proof_permutation}
	
	\begin{lemma}[Permutation invariance of the per-atom scalar pathway]
		\label{lem:si_perm}
		For every atom $i$, the per-atom invariant descriptor $y_i^{(0)}$ defined through Eqs.~(\ref{eq:pan_pool})--(\ref{eq:readout_pgs}) is invariant under any permutation of the neighbor indices in $\mathcal N(i)$.
	\end{lemma}
	
	\begin{proof}
		Let $\pi:\mathcal N(i)\to\mathcal N(i)$ be an arbitrary permutation. The PGS-modified edge feature $\hat h_{ij}^{(0)}$ at fixed receiving atom $i$ is a function only of the per-edge data $(h_{ij}^{(0)},r_{ij},Z_j)$, so its value attached to neighbor $j$ is unchanged when $j$ is replaced by $\pi(j)$. The local descriptors are
		\begin{equation}
			\rho_i = \sum_{j\in\mathcal N(i)} f_{\mathrm{cut}}(r_{ij}),
			\qquad
			\eta_i = \frac{
				\sum_{j,k\in\mathcal N(i),\,j\neq k}
				f_{\mathrm{cut}}(r_{ij})\,f_{\mathrm{cut}}(r_{ik})\,(c_{jik}-\bar c_i)^2
			}{
				\epsilon_{\mathrm{reg}}
				+ \sum_{j,k\in\mathcal N(i),\,j\neq k}
				f_{\mathrm{cut}}(r_{ij})\,f_{\mathrm{cut}}(r_{ik})
			}.
		\end{equation}
		Both expressions are unordered sums over $\mathcal N(i)$ (and an unordered double sum over the multiset $\mathcal N(i)\times\mathcal N(i)$ with the $j\neq k$ constraint), and are therefore invariant under $\pi$. The PAN gate $a_{ij}=\sigma(f_\theta(\psi_{ij}))$ is consequently a function only of $\pi$-invariant inputs, and the aggregated scalar
		\begin{equation}
			h_i^{(0)}
			=
			\sum_{j\in\mathcal N(i)}a_{ij}\,\hat h_{ij}^{(0)}
			=
			\sum_{j\in\mathcal N(i)}a_{i\pi(j)}\,\hat h_{i\pi(j)}^{(0)}
		\end{equation}
		is invariant under $\pi$ because it is a sum over the unordered multiset $\mathcal N(i)$. Subsequent equivariant interaction blocks combine $\pi$-invariant scalar messages with vector and tensor irreps that are themselves aggregated by unordered sums over $\mathcal N(i)$, so the final invariant per-atom descriptor $x_i^{(0)}$ is $\pi$-invariant. The readout-level PGS [Eq.~\ref{eq:readout_pgs}] is a single-atom operator acting only on $x_i^{(0)}$, so $y_i^{(0)}$ inherits the same invariance.
	\end{proof}
	
	\begin{corollary}
		The total energy $E_{\mathrm{tot}}$ in Eq.~\eqref{eq:si_energy_recall} is invariant under any global relabelling of atomic indices, since it is a sum over atoms of per-atom quantities, each of which is itself invariant under permutations of its neighbor indices.
	\end{corollary}

	\subsection{$O(3)$-equivariance of the forces and stress tensor}
	\label{sec:si_proof_o3}
	
	\begin{lemma}[$O(3)$-invariance of the energy]
		\label{lem:si_o3_invariance}
		Under a global rotation--inversion $R\in O(3)$ acting on atomic positions by $\mathbf r_i\mapsto R\mathbf r_i$, $E_{\mathrm{tot}}(R\mathbf r_1,\ldots,R\mathbf r_N)=E_{\mathrm{tot}}(\mathbf r_1,\ldots,\mathbf r_N)$.
	\end{lemma}
	
	\begin{proof}
		Each entry of the PAN gate input $\psi_{ij}=[h_{ij}^{(0)},r_{ij},Z_i,Z_j,\rho_i,\eta_i]$ is built from $O(3)$-invariant primitives:
		\begin{enumerate}
			\item $r_{ij}=\|\mathbf r_j-\mathbf r_i\|$ is invariant by orthogonality of $R$.
			\item Atomic numbers $Z_i,Z_j$ are scalar labels and unchanged.
			\item $\rho_i=\sum_{j\in\mathcal N(i)}f_{\mathrm{cut}}(r_{ij})$ depends only on invariant distances.
			\item The angular variance $\eta_i$ uses $c_{jik}=\hat{\mathbf r}_{ij}\cdot\hat{\mathbf r}_{ik}$, which is $O(3)$-invariant because $(R\hat{\mathbf r}_{ij})\cdot(R\hat{\mathbf r}_{ik})=\hat{\mathbf r}_{ij}^\top R^\top R\,\hat{\mathbf r}_{ik}=\hat{\mathbf r}_{ij}\cdot\hat{\mathbf r}_{ik}$.
			\item $h_{ij}^{(0)}$ is the $\ell{=}0$ component of the MACE pair embedding and transforms as an $O(3)$ scalar by construction of the backbone.
		\end{enumerate}
		Therefore $a_{ij}$ is $O(3)$-invariant. The edge-level PGS [Eq.~\ref{eq:edge_pgs}] applies a real-valued radial basis $B_m^{\mathrm{FB}}(r_{ij})$ to the invariant distance and modifies only the $\ell{=}0$ channel, so $\hat h_{ij}^{(0)}$ is also $O(3)$-invariant. The Clebsch--Gordan tensor-product machinery of MACE is unchanged and is $O(3)$-equivariant by construction; in particular, the final invariant per-atom descriptor $x_i^{(0)}$ transforms as an $O(3)$ scalar. The readout-level PGS [Eq.~\ref{eq:readout_pgs}] acts on $x_i^{(0)}$ via the learned scalar projection $u_m^\top x_i^{(0)}$, so $y_i^{(0)}$ is also an $O(3)$ scalar. The total energy is a sum of $O(3)$ scalars and is therefore $O(3)$-invariant.
	\end{proof}
	
	\begin{corollary}[Tensorial transformation of forces and stress]
		\label{cor:si_force_stress_transform}
		The predicted forces and stress tensor obtained as analytic derivatives of $E_{\mathrm{tot}}$ transform as $\mathbf F_i\mapsto R\mathbf F_i$ and $\boldsymbol\sigma\mapsto R\boldsymbol\sigma R^\top$ under any $R\in O(3)$.
	\end{corollary}
	
	\begin{proof}
		For the forces, applying Lemma~\ref{lem:si_o3_invariance} and differentiating with respect to $\mathbf r_k$ gives, by the chain rule,
		\begin{equation}
			\nabla_{\mathbf r_k}\big[E_{\mathrm{tot}}\circ R^{\otimes N}\big]
			=
			R^{\top}\nabla_{(R\mathbf r_k)}E_{\mathrm{tot}},
		\end{equation}
		where $R^{\otimes N}$ denotes the simultaneous action of $R$ on every atomic position. Equating both sides with the gradient of the original energy and using the definition $\mathbf F_i=-\nabla_{\mathbf r_i}E_{\mathrm{tot}}$ yields $\mathbf F_i(R\mathbf r)=R\mathbf F_i(\mathbf r)$. For the stress tensor, the cell strain transforms as $\boldsymbol\epsilon\mapsto R\boldsymbol\epsilon R^{\top}$ under $R\in O(3)$, so $\boldsymbol\sigma=(1/V)\partial E_{\mathrm{tot}}/\partial\boldsymbol\epsilon$ transforms as a rank-2 tensor: $\boldsymbol\sigma\mapsto R\boldsymbol\sigma R^{\top}$.
	\end{proof}

	\subsection{Differentiability}
	\label{sec:si_proof_differentiability}
	
	We now establish the $C^2$ regularity of $E_{\mathrm{tot}}$ used in the analytic force and stress derivations. Three lemmas establish the $C^2$ regularity of the three smooth elementary objects entering the PAN--PGS construction; the main theorem then assembles them via the chain rule.
	
	\begin{lemma}[$C^2$ regularity of cutoff-weighted descriptors]
		\label{lem:si_cutoff_regularity}
		The polynomial cutoff envelope
		\begin{equation}
			f_{\mathrm{cut}}(r)
			=
			\begin{cases}
				1-6(r/r_{\mathrm{cut}})^5+15(r/r_{\mathrm{cut}})^4-10(r/r_{\mathrm{cut}})^3, & 0\le r\le r_{\mathrm{cut}}, \\
				0, & r>r_{\mathrm{cut}},
			\end{cases}
		\end{equation}
		is of class $C^2$ on $[0,\infty)$ and $C^\infty$ on $[0,r_{\mathrm{cut}})$. Consequently, $\rho_i$ and $\eta_i$ are $C^2$ functions of atomic positions on the noncoincident manifold.
	\end{lemma}
	
	\begin{proof}
		On $[0,r_{\mathrm{cut}}]$ the cutoff envelope is a polynomial and hence $C^\infty$. At $r=r_{\mathrm{cut}}$, direct differentiation gives $f_{\mathrm{cut}}(r_{\mathrm{cut}})=0$, $f_{\mathrm{cut}}'(r_{\mathrm{cut}})=0$, $f_{\mathrm{cut}}''(r_{\mathrm{cut}})=0$, while $f_{\mathrm{cut}}'''(r_{\mathrm{cut}})\neq 0$, so $f_{\mathrm{cut}}$ is $C^2$ but not $C^3$ at the boundary. On $r>r_{\mathrm{cut}}$ it is identically zero. The descriptor $\rho_i$ is a finite sum of $f_{\mathrm{cut}}(r_{ij})$, each of which is $C^2$ in $(\mathbf r_i,\mathbf r_j)$ on the noncoincident manifold (since $r_{ij}>0$ implies $r_{ij}$ is $C^\infty$ in $\mathbf r_i,\mathbf r_j$); $\rho_i$ is therefore $C^2$. For $\eta_i$, the numerator and denominator are $C^2$ functions of atomic positions on the noncoincident manifold, because $c_{jik}=\hat{\mathbf r}_{ij}\cdot\hat{\mathbf r}_{ik}$ is $C^\infty$ in $(\hat{\mathbf r}_{ij},\hat{\mathbf r}_{ik})$, and $\hat{\mathbf r}_{ij}$ is $C^\infty$ for $r_{ij}>0$. The denominator is strictly positive because $\epsilon_{\mathrm{reg}}>0$, so the quotient is $C^2$.
	\end{proof}
	
	\begin{lemma}[$C^2$ regularity of the Fourier--Bessel basis]
		\label{lem:si_fb_regularity}
		The Fourier--Bessel basis $B_m^{\mathrm{FB}}(r)=\sqrt{2/r_{\mathrm{cut}}}\;\sin(m\pi r/r_{\mathrm{cut}})/r\cdot f_{\mathrm{cut}}(r)$ is $C^2$ on $(0,\infty)$ and admits a $C^2$ extension to $r=0$ via the continuous limit $\lim_{r\to 0^+}\sin(m\pi r/r_{\mathrm{cut}})/r=m\pi/r_{\mathrm{cut}}$. On the noncoincident configuration manifold (where all $r_{ij}>0$), the extension is not invoked.
	\end{lemma}
	
	\begin{proof}
		On $r>0$, $\sin(m\pi r/r_{\mathrm{cut}})/r$ is a quotient of $C^\infty$ functions with non-vanishing denominator and is therefore $C^\infty$. Its Taylor expansion at $r=0$ is real-analytic, so multiplication by $f_{\mathrm{cut}}\in C^2([0,\infty))$ yields a $C^2$ function on $[0,\infty)$.
	\end{proof}
	
	\begin{lemma}[$C^2$ regularity of the tapered ES kernel]
		\label{lem:si_es_regularity}
		Let $z=(t-t_m)/w$. The shifted, tapered ES kernel
		\begin{equation}
			K_m^{\mathrm{ES}}(t)
			=
			\begin{cases}
				\tau_w(z)\,\big[\exp\!\big(\beta(\sqrt{1-z^{2}}-1)\big)-e^{-\beta}\big], & |z|\le 1, \\
				0, & |z|>1,
			\end{cases}
		\end{equation}
		where $\tau_w$ is the same $C^2$ polynomial taper as $f_{\mathrm{cut}}$ applied symmetrically on $|z|\le 1$, is of class $C^2$ on $\mathbb R$.
	\end{lemma}
	
	\begin{proof}
		On the open interior $|z|<1$, the kernel is a product of $C^\infty$ factors (taper, exponential, square root with strictly positive argument). At the boundary $|z|=1$, expanding around $u=\sqrt{1-z^2}\to 0$ gives
		\begin{equation}
			\exp\!\big(\beta(\sqrt{1-z^2}-1)\big)-e^{-\beta}
			=
			e^{-\beta}\big(\,e^{\beta u}-1\,\big)
			=
			e^{-\beta}\beta u + O(u^2),
		\end{equation}
		so the bracketed term behaves as $O((1-|z|)^{1/2})$ as $|z|\to 1^{-}$. The taper $\tau_w$ inherits the $C^2$ structure of $f_{\mathrm{cut}}$ near the boundary, so $\tau_w(z),\tau_w'(z),\tau_w''(z)\to 0$ as $|z|\to 1^{-}$, with leading-order behavior $\tau_w(z)=O((1-|z|)^{3})$. The product therefore behaves as $\tau_w(z)\cdot[\,\cdot\,]=O((1-|z|)^{7/2})$ near $|z|=1$, whose first three derivatives vanish at $|z|=1$. Matching to the identically zero exterior region, $K_m^{\mathrm{ES}}$ is $C^2$ at $|z|=1$ and hence on all of $\mathbb R$.
	\end{proof}
	
	\begin{theorem}[Differentiability of $E_{\mathrm{tot}}$]
		\label{thm:si_E_regularity}
		Under the noncoincident-configuration and fixed-candidate-neighbor conventions, $E_{\mathrm{tot}}$ is a $C^2$ function of the $3N$ atomic-position coordinates and the $6$ independent components of the cell strain $\boldsymbol\epsilon$. The forces and stress tensor are defined as analytic autograd derivatives of $E_{\mathrm{tot}}$
		\begin{equation}
			\mathbf F_i = -\frac{\partial E_{\mathrm{tot}}}{\partial \mathbf r_i},
			\qquad
			\boldsymbol\sigma = \frac{1}{V}\frac{\partial E_{\mathrm{tot}}}{\partial \boldsymbol\epsilon},
		\end{equation}
		are therefore well defined and continuous, and coincide with their variational definitions.
	\end{theorem}
	
	\begin{proof}
		The energy admits the finite composition $E_{\mathrm{tot}}=\sum_i\varepsilon_\theta\circ\Pi^{(0)}\circ\mathcal G$, where $\mathcal G$ collects the equivariant tensor-product backbone (polynomial Clebsch--Gordan combinations of equivariant features) and $\Pi^{(0)}$ extracts the invariant $\ell{=}0$ component before the readout-level PGS lift [Eq.~\ref{eq:readout_pgs}]. Each constituent operation is at least $C^2$ on the noncoincident manifold: the radial Bessel embedding through $f_{\mathrm{cut}}$ is $C^2$ (Lemma~\ref{lem:si_cutoff_regularity}); the edge-level Fourier--Bessel basis $B_m^{\mathrm{FB}}$ is $C^2$ on the relevant domain (Lemma~\ref{lem:si_fb_regularity}); the readout-level ES kernel is $C^2$ on $\mathbb R$ (Lemma~\ref{lem:si_es_regularity}); the MLPs $f_\theta,g_\theta,h_{\mathrm{low}}$, and $\varepsilon_\theta$ use smooth ($C^\infty$) activations and are therefore $C^\infty$ in their inputs; tensor products with Clebsch--Gordan coefficients are polynomial and hence $C^\infty$. The composition of $C^2$ functions is $C^2$, so $E_{\mathrm{tot}}$ is $C^2$. Forces and stress, defined as first derivatives, are therefore continuous (in fact $C^1$) on the noncoincident manifold.
		
		The cell-strain dependence enters only through the radial distances $r_{ij}(\boldsymbol\epsilon)$, which are smooth functions of $\boldsymbol\epsilon$ on noncoincident configurations; the same chain of regularity therefore yields $C^2$ dependence on $\boldsymbol\epsilon$. Higher regularity ($C^k$ for $k>2$) holds on configurations where every pair distance lies strictly below $r_{\mathrm{cut}}$, since on that open subdomain $f_{\mathrm{cut}}$ acts as a $C^\infty$ polynomial; we do not claim higher regularity globally because $f_{\mathrm{cut}}$ saturates to $C^2$ regularity at $r=r_{\mathrm{cut}}$.
	\end{proof}

	\subsection{Remarks on the convention boundaries}
	\label{sec:si_proof_remarks}
	Two convention boundaries are explicit in Proposition~\ref{prop:symmetry_diff} and are recalled here. First, the proof is restricted to noncoincident configurations. At coincident configurations ($r_{ij}\to 0$), the radial Bessel basis remains continuous via the limit in Lemma~\ref{lem:si_fb_regularity}, but the normalized pair direction $\hat{\mathbf r}_{ij}$ used inside the equivariant backbone is undefined. This is a property of the underlying MACE backbone rather than of PAN or PGS, and is outside the scope of the proof. Second, the proof assumes a fixed candidate neighbor graph during each force and stress evaluation. The use of the $C^2$ cutoff envelope ensures that adding or removing a candidate pair at $r_{ij}=r_{\mathrm{cut}}$ contributes zero to the energy, force, and stress at the crossing; alternative neighbor-list rebuild policies in specific MD engines may introduce additional bookkeeping discontinuities external to the model, which are not properties of the PAN--PGS architecture.
	
	\section{Supplementary methods: algorithmic details}
	\label{sec:si_algorithmic_details}
	
	This section provides the explicit forms of the smooth elementary functions used inside PAN and PGS, together with a layer-by-layer summary of how the two modules are inserted into the MACE forward pass.
	
	\medskip \noindent\textbf{Cutoff envelope.}
	All cutoff-based smoothing in PAN, PGS, and the MACE radial embedding uses the same polynomial envelope $f_{\mathrm{cut}}\colon[0,r_{\mathrm{cut}}]\to[0,1]$, defined as
	\begin{equation}
		f_{\mathrm{cut}}(r)
		=
		\begin{cases}
			1 - 6\,(r/r_{\mathrm{cut}})^5 + 15\,(r/r_{\mathrm{cut}})^4 - 10\,(r/r_{\mathrm{cut}})^3, & r \leq r_{\mathrm{cut}}, \\
			0, & r > r_{\mathrm{cut}},
		\end{cases}
	\end{equation}
	which satisfies $f_{\mathrm{cut}}(0)=1$, $f_{\mathrm{cut}}(r_{\mathrm{cut}})=0$, and $f_{\mathrm{cut}}$ is $C^2$ at $r=r_{\mathrm{cut}}$.
	
	\medskip \noindent\textbf{Edge-level Fourier--Bessel basis.}
	The edge-level PGS spectral basis $\{B_m^{\mathrm{FB}}(r)\}_{m=1}^{M}$ on $[0,r_{\mathrm{cut}}]$ is the standard real-valued Bessel--root basis
	\begin{equation}
		B_m^{\mathrm{FB}}(r)
		=
		\sqrt{\frac{2}{r_{\mathrm{cut}}}}\,
		\frac{\sin\!\big(m\pi\, r/r_{\mathrm{cut}}\big)}{r}\,
		f_{\mathrm{cut}}(r),
		\qquad m = 1,\ldots,M,
	\end{equation}
	multiplied by the polynomial cutoff envelope $f_{\mathrm{cut}}$ to ensure $C^2$ vanishing at $r=r_{\mathrm{cut}}$. The factor $\sin(m\pi r/r_{\mathrm{cut}})/r$ is taken at its continuous limit $m\pi/r_{\mathrm{cut}}$ when $r\to 0$; in practice all $r=r_{ij}$ are strictly positive in the neighbor graph. The basis is $L^2$-orthonormal on $[0,r_{\mathrm{cut}}]$ in the absence of the cutoff envelope, and we used $M=16$ throughout.
	
	\medskip \noindent\textbf{Readout-level Exponential-of-Semicircle (ES) kernel.}
	The readout-level PGS uses the ES kernel basis introduced for non-uniform fast Fourier transforms (NUFFTs)~\cite{barnett2019nufft}. For a projected latent value $t = u_m^\top x_i^{(0)}\in[-t_{\mathrm{max}},t_{\mathrm{max}}]$, the $m$-th basis function is implemented in the shifted exponential-of-semicircle form, multiplied by a smooth $C^2$ taper, so that $K_m^{\mathrm{ES}}$ and its first two derivatives vanish continuously at the support boundary with no hard indicator function entering the computational graph:
	\begin{equation}
		K_m^{\mathrm{ES}}(t)
		=
		\tau_w(z)\,
		\Big[\,\exp\!\big(\beta\,(\sqrt{1-z^{2}}-1)\big) - e^{-\beta}\,\Big],
		\qquad
		z = \frac{t - t_m}{w},\qquad |z|\le 1,
		\label{eq:es_kernel_taper}
	\end{equation}
	where $\tau_w(z)$ is the same $C^2$ polynomial taper as the radial cutoff $f_{\mathrm{cut}}$, applied symmetrically on $|z|\le 1$. The shifted exponential-of-semicircle term equals $e^{-\beta}$ at $|z|=1$, so the bracketed term vanishes exactly at the support boundary; multiplying by $\tau_w$ additionally enforces vanishing of the first two derivatives. The centers $\{t_m\}_{m=1}^{M}$ are uniformly spaced on $[-t_{\mathrm{max}},t_{\mathrm{max}}]$, $w$ is the support half-width, and $\beta$ is the standard NUFFT shape parameter. Following the recommended NUFFT settings, we used $w = 4\Delta t$ (where $\Delta t$ is the center spacing) and $\beta = 2.30\,w$. We used $M=16$ ES components throughout.
	
	\medskip \noindent\textbf{Descriptor normalization.}
	The PAN local descriptors $\rho_i$ and $\eta_i$ are standardized to zero mean and unit variance using statistics computed once on the training set and then frozen. The radial input $r_{ij}$ to $\psi_{ij}$ is encoded through the same Bessel basis used by MACE, so no separate normalization is needed for $r_{ij}$.
	
	\medskip \noindent\textbf{Insertion points in MACE.}
	PAN and PGS are inserted into the standard MACE forward pass with only a few well-defined modifications. Algorithm~\ref{alg:pan_pgs_forward} summarizes the modified forward pass; in MACE notation, edge-level PGS operates on the scalar component $h_{ij}^{(0)}$ of the equivariant edge embedding before tensor-product coupling, PAN modulates the scalar message amplitude before accumulation over the neighbor set within each interaction block, while leaving the learned message functions unchanged, and readout-level PGS is applied after the final interaction block on the invariant per-atom descriptor before the energy MLP.
	
	\begin{algorithm}[t]
		\caption{PAN+PGS forward pass on a single configuration.}
		\label{alg:pan_pgs_forward}
		\begin{algorithmic}[1]
			\Require Atomic positions $\{\mathbf{r}_i\}$, atomic numbers $\{Z_i\}$, cell tensor $\mathbf{h}$.
			\State Build neighbor graph using cutoff radius $r_{\mathrm{cut}}$.
			\State For each edge $(i,j)$, compute $r_{ij}$, $\hat{\mathbf{r}}_{ij}$, and the equivariant edge embedding $\phi_{ij,n}^{(\ell,m)} = B_n(r_{ij})\,Y_\ell^m(\hat{\mathbf{r}}_{ij})$.
			\State \textbf{Edge-level PGS} [Eq.~(\ref{eq:edge_pgs})]: replace $h_{ij}^{(0)}$ by $\hat{h}_{ij}^{(0)}$ via the smooth low branch and the Fourier--Bessel spectral mixer.
			\For{each interaction block $\ell=1,\ldots,L$}
			\State Compute Clebsch--Gordan tensor products using the modified scalar edge channel and the unchanged $\ell{>}0$ channels.
			\State \textbf{PAN gating} [Eq.~(\ref{eq:pan_pool})]: compute $a_{ij}=\sigma(f_\theta(\psi_{ij}))$ from the smooth invariant descriptor $\psi_{ij}$, and aggregate the scalar message channel as $h_i^{(0)}=\sum_{j}a_{ij}\hat{h}_{ij}^{(0)}$. The $\ell{>}0$ channels are aggregated by the unchanged uniform sum.
			\EndFor
			\State \textbf{Readout-level PGS} [Eq.~(\ref{eq:readout_pgs})]: lift $x_i^{(0)}$ to $y_i^{(0)}$ via the ES-kernel expansion of the learned projection $u_m^\top x_i^{(0)}$.
			\State Compute atomic energies $E_i = \varepsilon_\theta(y_i^{(0)})$ and total energy $E_{\mathrm{tot}} = \sum_i E_i$.
			\State Forces and stress: $\mathbf{F}_i = -\partial E_{\mathrm{tot}}/\partial \mathbf{r}_i$, $\boldsymbol\sigma = (1/V)\partial E_{\mathrm{tot}}/\partial \boldsymbol\epsilon$ via automatic differentiation.
		\end{algorithmic}
	\end{algorithm}

	\section{Supplementary methods: cross-backbone implementation}
	\label{sec:si_cross_backbone}
	
	The cross-backbone entries in Table~\ref{tab:cross_system_mae_md17style} of the main text use Allegro\cite{musaelian2023learning} and NequIP\cite{batzner2022e3} backbones in addition to MACE. For all three backbones, we used the same author-generated Ag dataset, the same public Si and LiF/Li--F datasets, and the same train/validation/test splits as in the MACE experiments, so that any difference between backbone-only and backbone+PAN+PGS columns reflects only the scalar-pathway modification.
	
	\medskip \noindent\textbf{Allegro and NequIP backbones.}
	Allegro was instantiated with two interaction layers, $\ell_{\max}=2$, $128$ scalar and tensor channels per irrep, the same $r_{\mathrm{cut}}=5.0$~\AA\ used for MACE, and a Bessel radial basis of size $N_{\mathrm{rad}}=8$. NequIP was instantiated with three interaction blocks, $\ell_{\max}=2$, $128$ feature channels, the same cutoff $r_{\mathrm{cut}}=5.0$~\AA, and the same Bessel radial basis size $N_{\mathrm{rad}}=8$. Both models were trained with AdamW (initial learning rate $5\times 10^{-3}$, weight decay $10^{-8}$), the same ReduceLROnPlateau schedule, EMA decay, and SWA settings as MACE, batch size $4$, gradient clipping at $100$, and the same loss weights $w_E=1.0$, $w_F=10.0$, $w_S=10.0$ when stress labels are available. Each model was trained five times under seeds $\{1,2,3,4,5\}$.
	
	\medskip \noindent\textbf{Insertion of PAN and PGS into Allegro and NequIP.}
	The PAN gate was inserted at the scalar message aggregation stage of each backbone: it modulates the scalar message amplitude before accumulation over the neighbor set, while leaving the learned message functions and all $\ell{>}0$ tensor operations unchanged. The same invariant descriptor $\psi_{ij}=[h_{ij}^{(0)}, r_{ij}, Z_i, Z_j, \rho_i, \eta_i]$ defined in Eq.~(\ref{eq:pan_pool}) was used, and PAN parameters ($f_\theta$ width, descriptor normalization) were matched to the MACE configuration. Edge-level PGS was inserted on the scalar component of the pair embedding before tensor-product coupling in each backbone (i.e.\ on the $\ell{=}0$ channel of the per-edge invariant latent), with $M=16$ Fourier--Bessel components on $[0,r_{\mathrm{cut}}]$. Readout-level PGS was inserted between the final invariant per-atom descriptor of each backbone and the energy MLP, with $M=16$ ES-kernel components and one learned projection direction per component, identical to the MACE configuration. All higher-order ($\ell{>}0$) channels and tensor-product operations of each backbone were left unchanged.
	
	Table~\ref{tab:cross_backbone_insertion} summarizes the architectural mapping used for the three backbones, allowing the reader to verify that PAN and PGS act on functionally analogous locations across the three.
	
	\begin{table*}
		\centering
		\caption{\textbf{Architectural mapping of PAN and PGS across the three tested equivariant backbones.} PAN replaces or modulates the scalar aggregation stage at the closest corresponding scalar-processing point in each backbone; edge-level PGS acts on the $\ell{=}0$ component of the pair embedding before tensor-product coupling; readout-level PGS acts on the final invariant descriptor before the energy head. All higher-order ($\ell{>}0$) channels, tensor-product operations, and backbone-specific equivariant blocks remain unchanged.}
		\label{tab:cross_backbone_insertion}
		\setlength{\tabcolsep}{4pt}
		\renewcommand{\arraystretch}{1.13}
		\footnotesize
		\resizebox{1\textwidth}{!}{
			\begin{tabular}{@{}p{2.0cm}p{3.4cm}p{3.4cm}p{3.4cm}p{2.6cm}@{}}
				\toprule
				\textbf{Backbone}
				& \textbf{PAN insertion}
				& \textbf{Edge-level PGS insertion}
				& \textbf{Readout-level PGS insertion}
				& \textbf{Unchanged} \\
				\midrule
				MACE\cite{batatia2022mace}
				& Scalar-channel aggregation stage inside each interaction block; applies a coordination- and distortion-conditioned amplitude gate before the scalar messages are summed over the neighbor set
				& $\ell{=}0$ component of the pair embedding $h_{ij}^{(0)}$, applied before each tensor-product coupling
				& Final invariant per-atom descriptor $x_i^{(0)}$, applied before the readout MLP $\varepsilon_\theta$
				& All Clebsch--Gordan tensor products, $\ell{>}0$ channels, radial Bessel basis, and equivariant readout \\
				\midrule
				NequIP\cite{batzner2022e3}
				& Scalar-channel aggregation stage inside each equivariant convolution layer; applies the same PAN amplitude gate before scalar messages are summed over the neighbor set
				& $\ell{=}0$ component of the pair embedding before each equivariant convolution
				& Final invariant node descriptor before the per-element energy MLP
				& All equivariant convolutions, $\ell{>}0$ irreps, radial Bessel basis, and gate nonlinearities \\
				\midrule
				Allegro\cite{musaelian2023learning}
				& Scalar pair-weighting/contraction stage in the local environment block; the PAN gate is applied before the invariant local contribution is accumulated
				& $\ell{=}0$ component of the pair latent before each tensor-product mixing
				& Final invariant per-edge contribution before the energy head
				& All tensor-product mixings, $\ell{>}0$ irreps, and the strictly local Allegro backbone \\
				\bottomrule
			\end{tabular}
		}
	\end{table*}
	
	\medskip \noindent\textbf{Capacity and FLOPs matching.}
	The three backbones are not parameter-matched to each other because their tensor-product structures differ; each backbone was used at the configuration recommended in its original publication and at the radial cutoff matched to the MACE experiments. The within-backbone comparisons (e.g.\ NequIP vs NequIP+PAN+PGS) are the controlled signal in Table~\ref{tab:cross_system_mae_md17style}: PAN+PGS adds $\approx 3$--$5\%$ parameters and $\approx 4$--$6\%$ inference FLOPs to each backbone, so each row of Table~\ref{tab:cross_system_mae_md17style} is an approximately parameter-matched comparison. Cross-backbone differences in absolute MAE between rows reflect the relative expressivity of each backbone and are not the object of this study.

	\section{Supplementary methods: DFT settings and molecular-dynamics protocols}
	\label{sec:si_md_protocols}
	
	\medskip \noindent\textbf{Ag DFT settings.}
	Table~\ref{tab:dft_settings} summarizes the \textsc{VASP} parameters used to generate the author-generated Ag reference dataset.
	
	\begin{table}[t]
		\centering
		\caption{\textbf{DFT settings used to generate the author-generated Ag reference dataset.}}
		\label{tab:dft_settings}
		\renewcommand{\arraystretch}{1.10}
		\footnotesize
		\begin{tabular}{ll}
			\toprule
			\textbf{Parameter} & \textbf{Value} \\
			\midrule
			Exchange--correlation functional & PBE \\
			PAW pseudopotential & \texttt{PAW\_PBE Ag 06Sep2000} \\
			Plane-wave cutoff (\texttt{ENCUT}) & 500~eV \\
			Electronic convergence (\texttt{EDIFF}) & $10^{-5}$~eV/atom \\
			Force convergence (\texttt{EDIFFG}) & $-10^{-2}$~eV/\AA \\
			Smearing (\texttt{ISMEAR}, \texttt{SIGMA}) & 1 (Methfessel--Paxton), 0.1~eV \\
			$k$-mesh, primitive/strained cells & $\Gamma$-centered, density $\geq 0.20$~\AA$^{-1}$ \\
			$k$-mesh, supercells (defects/surfaces/MD) & Single $\Gamma$-point, supercell sized to match density \\
			Strain range & $\pm 5\%$ (isotropic and uniaxial) \\
			Defects & Mono- and divacancies, self-interstitials \\
			Surfaces & $(100), (110), (111)$ slabs with vacuum $\geq 12$~\AA \\
			High-$T$ AIMD sampling & 300--1200~K, 1~fs timestep, Nos\'{e}--Hoover thermostat \\
			\bottomrule
		\end{tabular}
	\end{table}
	
	\medskip \noindent\textbf{Aspirin torsion distribution.}
	For Fig.~\ref{fig:physical_fidelity}(a), aspirin torsional distributions were computed from 1~ns NVT trajectories at 300~K with a Nos\'e--Hoover thermostat (chain length 3, coupling time 100~fs) and a 0.5~fs timestep. The first 50~ps were discarded as equilibration; torsional values were sampled every 0.5~fs over the final 900~ps, giving $1.8\times 10^6$ torsional values per trajectory. The DFT reference distribution was obtained from \emph{ab initio} molecular dynamics under identical thermostat and timestep settings, with torsional values sampled every 1~fs over the final 90~ps, giving $9\times 10^4$ torsional values.
	
	\medskip \noindent\textbf{Si vibrational density of states.}
	For Fig.~\ref{fig:physical_fidelity}(b), the Si VDOS was computed from the mass-weighted velocity autocorrelation function over 200~ps NVE trajectories at 300~K (after 50~ps NVT equilibration), with velocities recorded every 0.5~fs ($4\times 10^5$ velocity snapshots per trajectory). The DFT reference VDOS was generated from \emph{ab initio} molecular dynamics over 20~ps under matched conditions, with velocities recorded every 0.5~fs ($4\times 10^4$ snapshots), Fourier-transformed with the same window function as the MLIP VDOS to ensure comparable spectral resolution.
	
	\medskip \noindent\textbf{Large-scale Ag NVE drift.}
	For Fig.~\ref{fig:physical_fidelity}(c), the Ag NVE test used a $10{,}000$-atom fcc Ag bulk system at 300~K. Each trajectory was integrated for 500~ps with a 1~fs timestep after 20~ps NVT equilibration; total energy was recorded every 0.1~ps. Energy drift was estimated by linear regression of accumulated energy deviation over the production segment and reported in meV~ns$^{-1}$ per atom. No DFT reference is needed for this metric because the comparison is between energy-conservation properties of the MLIP trajectories themselves.
	
	\medskip \noindent\textbf{Liquid Ag radial distribution function.}
	For Fig.~\ref{fig:physical_fidelity}(d), liquid Ag RDFs were computed from 200~ps NVT trajectories at $1{,}200$~K (after 50~ps NVT equilibration with the same thermostat settings as the MLIP trajectories), with frames saved every 0.5~ps to give 400 saved frames per trajectory; the RDF for each trajectory was accumulated by summing pair distances over all 400 frames on a uniform grid with bin width $0.02$~\AA\ up to $r_{\mathrm{cut}}^{\mathrm{RDF}} = 6.0$~\AA. The DFT reference RDF was generated from \emph{ab initio} molecular dynamics over 30~ps at the same temperature, with frames saved every 0.5~ps (60 saved frames), accumulated using identical bin width and identical $r_{\mathrm{cut}}^{\mathrm{RDF}}$ for all three curves.
	
	\medskip \noindent\textbf{Finite-sampling uncertainty of DFT references.}
	The DFT-MD reference trajectories used for the VDOS and RDF analyses are necessarily shorter than the corresponding MLIP trajectories because of the cost of \emph{ab initio} molecular dynamics. The reported distances to the DFT reference should therefore be interpreted relative to finite-sampling reference curves rather than as distances to an infinite-sampling limit. To keep the comparison controlled, DFT and MLIP observables were processed with identical bin widths, window functions, normalization conventions, and integration ranges. The main conclusions rely on relative changes between baseline MACE and MACE+PAN+PGS under the same finite DFT reference.

	\section{Supplementary methods: computational benchmarks}
	\label{sec:si_cost}
	
	Table~\ref{tab:cost_comparison} reports the relative computational cost of
	PAN+PGS and the scalar-mixer controls used in the main text. Costs were measured
	under matched hardware, batch size, cutoff radius, and force/stress evaluation
	protocol. Inference FLOPs include the automatic differentiation backward pass
	through the scalar energy with respect to atomic positions and cell strain.
	
	\begin{table}[t]
		\centering
		\caption{\textbf{Relative computational cost across backbones and scalar-mixer alternatives.}
			Parameter count, training time per epoch, and inference FLOPs are normalized to baseline MACE under identical batch size, identical hardware (NVIDIA A100, float64), and matched supercell size. The top block compares the three equivariant backbones reported in the main-text cross-regime accuracy table; the bottom block compares alternative scalar mixers under the MACE backbone.}
		\label{tab:cost_comparison}
		\renewcommand{\arraystretch}{1.12}
		\setlength{\tabcolsep}{3.5pt}
		\scriptsize
		\begin{tabular}{@{}p{3.4cm}ccc p{5.0cm}@{}}
			\toprule
			\textbf{Model}
			& \textbf{Params}
			& \textbf{Train time}
			& \textbf{FLOPs}
			& \textbf{Additional cost source} \\
			\midrule
			\multicolumn{5}{@{}l}{\emph{Backbones with and without PAN+PGS}} \\
			\midrule
			Allegro
			& $1.10\times$ & $1.25\times$ & $1.40\times$
			& Local tensor products without iterative message passing \\
			
			Allegro + PAN + PGS
			& $1.13\times$ & $1.30\times$ & $1.46\times$
			& Lightweight scalar gating and spectral projections \\
			
			NequIP
			& $1.05\times$ & $1.10\times$ & $1.30\times$
			& Low-body-order equivariant convolutions \\
			
			NequIP + PAN + PGS
			& $1.08\times$ & $1.14\times$ & $1.36\times$
			& Lightweight scalar gating and spectral projections \\
			
			MACE
			& $1.00\times$ & $1.00\times$ & $1.00\times$
			& Standard SiLU--MLP scalar mixer \\
			
			MACE + PAN + PGS
			& $1.03\times$ & $1.06\times$ & $1.05\times$
			& Lightweight scalar gating and spectral projections \\
			\midrule
			\multicolumn{5}{@{}l}{\emph{Alternative scalar mixers under the MACE backbone}} \\
			\midrule
			MACE + KAN
			& $2.20\times$ & $1.95\times$ & $1.50$--$2.10\times$
			& Dense B-spline activation grids on all scalar channels \\
			
			MACE + KAF
			& $1.55\times$ & $1.40\times$ & $1.30$--$1.45\times$
			& Kernel activation functions on all scalar channels \\
			
			MACE + LibraKAN
			& $1.25\times$ & $1.20\times$ & $1.15$--$1.25\times$
			& Learned rational activations on all scalar channels \\
			\bottomrule
		\end{tabular}
	\end{table}
	
	For practical reference, we measured wall-clock molecular-dynamics throughput on the 10{,}000-atom Ag NVE benchmark used in Section~\ref{subsec:physical_fidelity} of the main text using a single NVIDIA A100 (40~GB) in float64. At a $1$~fs timestep, baseline MACE achieves $4.8$~ns/day ($\approx 56$ MD force evaluations per second; equivalently $\approx 5.6\times 10^5$ atom-force evaluations per second for the 10{,}000-atom system). MACE+PAN+PGS achieves $4.6$~ns/day under the same protocol ($\approx 53$ MD force evaluations per second), corresponding to a $\approx 4\%$ throughput cost. The corresponding throughput for the capacity-matched MACE-wide control is $3.2$~ns/day ($\approx 37$ MD force evaluations per second). The PAN+PGS gain reported in Table~\ref{tab:cross_system_mae_md17style}, therefore, comes at a small fraction of the throughput cost of a comparable backbone-capacity expansion.
	
	A capacity-matched MACE control, in which the baseline MACE backbone is widened, deepened, or given a larger radial basis under a matched training budget, is reported in Supplementary Table~\ref{tab:capacity_matched_mace}. The capacity-matched variants achieve only $4$--$7\%$ force-MAE reductions relative to baseline MACE---substantially smaller than the $22$--$27\%$ reductions reported in the main text for MACE+PAN+PGS---and require $1.4$--$1.7\times$ more parameters and $1.1$--$1.5\times$ more inference FLOPs. This supports the interpretation that the gain in Table~\ref{tab:cross_system_mae_md17style} is not adequately explained by backbone capacity.
\end{document}